\documentclass[letterpaper]{article} 
\usepackage{aaai2026}  
\usepackage{times}  
\usepackage{helvet}  
\usepackage{courier}  
\usepackage[hyphens]{url}  
\usepackage{mathtools}        
\usepackage{amsthm}

\usepackage{cancel}           
\usepackage{nicefrac}
\usepackage{booktabs} 
\usepackage{pifont}  

\usepackage{algorithm}
\usepackage{algpseudocode}
\usepackage{subcaption}
\newtheorem{remark}{Remark}
\newtheorem{definition}{Definition}
\newtheorem{theorem}{Theorem}
\newtheorem{proposition}{Proposition}
\newtheorem{corollary}{Corollary}

\usepackage{xcolor} 

\newcommand{\Commentr}[1]{\hfill\makebox[0.65\linewidth][l]{\textit{\small// #1}}}

\usepackage{multirow}
\usepackage{graphicx} 
\usepackage{natbib}  
\usepackage{caption} 
\usepackage{newfloat}
\usepackage{listings}
\DeclareCaptionStyle{ruled}{labelfont=normalfont,labelsep=colon,strut=off} 
\floatstyle{ruled}
\newfloat{listing}{tb}{lst}{}
\floatname{listing}{Listing}
\title{Fair Finetuning Mitigates Distribution Inference Attacks}
\author {
    Rakshit Naidu
}
\affiliations{
    Georgia Institute of Technology\\
    rakshitnaidu@gatech.edu
}

\usepackage[
    n,
    operators,
    advantage,
    sets,
    adversary,
    landau,
    probability,
    notions,
    logic,
    ff,
    mm,
    primitives,
    events,
    complexity,
    asymptotics,
    keys
]{cryptocode}

\newcommand{\cmark}{\ding{51}}
\newcommand{\xmark}{\ding{55}}

\begin{document}

\maketitle

\begin{abstract}
Machine learning models trained on sensitive data can inadvertently
leak population-level information about their training
distributions---a threat known as a \emph{distribution inference
attack} (DIA). An adversary with only black-box access to a model can
infer sensitive demographic properties, such as the proportion of a
subgroup represented in training, without observing a single data
record directly. While defenses such as differential privacy and
property unlearning have been proposed, the connection between fairness
constraints and distributional leakage remains largely unexplored.

We propose Fair Fine-tuning (FFt) as a principled defense: a trained
model is fine-tuned on samples from the complementary distribution
subject to an Equalized Odds (EO) constraint. We develop a complete
theoretical characterisation. Our main result
(Theorem~\ref{thm:eqodds}) gives the tight bound
$\text{Adv}(\mathcal{A},M_f) \le \Delta_{\text{EO}} \cdot W$, where
$W$ is a computable distributional shift weight that quantifies how
distinguishable the two training distributions are by their
sensitive-attribute composition. We further prove a necessary condition
for FFt to reduce adversarial advantage
(Theorem~\ref{thm:condition}) and establish tightness of the bound
(Proposition~\ref{prop:tight}). We evaluate under a biased distribution protocol---pure demographic
groups as $\mathcal{G}_0$ and $\mathcal{G}_1$, so $W\!=\!1$ and
adversarial leakage is maximal---across six datasets spanning tabular
(ACS Income, COMPAS, German Credit), image (UTKFaces), and NLP (Bias
in Bios) modalities, with LSAC evaluated in the appendix.
We illustrate that rehearsal-based FFt consistently reduces the adversarial accuracy gap below the detection threshold $\tau\!=\!0.1$ in the majority of runs
across all settings and modalities. For instance, on ACS Income, the
gap falls from $\sim\!15\%$ to under $4\%$, displaying the effectiveness of our method.
Our work provides the first formal bound connecting a model's measured EO disparity directly to its adversarial advantage in the DIA game, opening a new avenue for unified fairness-and-privacy defenses against inference attacks.
\end{abstract}


\section{Introduction}

Machine learning models trained on sensitive data can inadvertently expose population-level properties of their training distribution to adversaries with only black-box access to the model's outputs. Earlier work on \emph{membership inference attacks}~\cite{shokrimia17, yeommembershipinfer18} established that querying a model can reveal whether a specific individual was in its training set. We focus on a more recent and less-studied threat: \emph{distribution inference attacks} (DIAs), in which an adversary recovers global properties of the training distribution — such as the proportion of a demographic subgroup, label priors, or correlations between sensitive attributes and outcomes — without observing any individual record directly.

As a concrete example, consider a model trained on a loan recidivism dataset that includes \texttt{gender} as an attribute. An adversary who can query the model may be able to infer not just that gender is present, but the precise male-to-female ratio in the training data. Such population-level leakage can have serious consequences: it reveals the demographic composition of a private dataset and may enable downstream discrimination or targeted manipulation.

Given these risks, a natural question arises: \emph{can training procedures that enforce fairness constraints reduce distributional leakage?} Fairness interventions — such as equalized odds penalties, adversarial debiasing, or reweighting — are designed to suppress the model's dependence on demographic structure. If distribution inference attacks exploit exactly these demographic cues, then enforcing fairness may implicitly limit the distributional information available to an adversary. Despite growing interest in both fairness and privacy, this connection has remained largely unexplored.

We answer the question formally. Our central result (Theorem~\ref{thm:eqodds}) proves that for any model $M_f$ fine-tuned with an Equalized Odds (EO) constraint, the adversary's advantage in the DIA game is bounded by $\Delta_{\text{EO}} \cdot W$, where $\Delta_{\text{EO}}$ is the EO disparity and $W$ is a computable distributional shift weight. To our knowledge, this is the first formal bound directly connecting an operationalised fairness metric (EO disparity) to adversarial advantage in the distribution inference game — prior formal frameworks such as attribute privacy~\citep{zhang22} and Pufferfish~\citep{kifer12} characterise leakage over data-generating mechanisms rather than yielding explicit, computable bounds on an adversary's distinguishing accuracy as a function of a model's measured fairness.

We propose \emph{Fair Fine-tuning (FFt)}: after training a baseline model on the base distribution $\mathcal{G}_0$, the defender then fine-tunes it on a sample from the complementary distribution $\mathcal{G}_1$ subject to an EO penalty, as shown in Figure~\ref{fig:defense}. The EO constraint is the operative ingredient — it forces the residual per-label unfairness $\delta_y$ to zero, cancelling the base prediction rate $P_{y,0}$ from the leakage expression entirely. A rehearsal component (replaying a fraction $\rho$ of $\mathcal{G}_0$ data during fine-tuning) prevents catastrophic forgetting. Throughout this work, the adversary is assumed to have only black-box access to the released model.

\begin{figure*}[t]
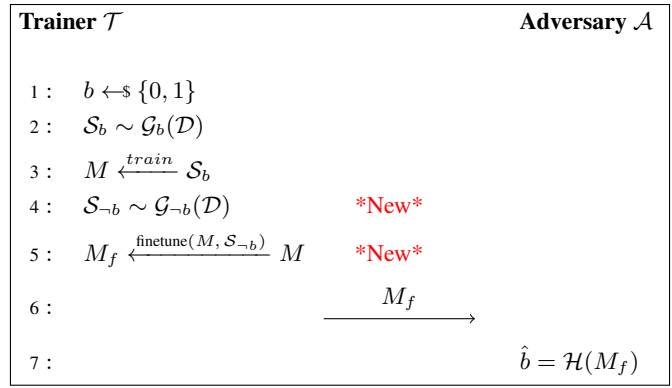

    \centering
    \fbox{%
    \pseudocode[]{%
        \textbf{Trainer } \mathcal{T} \<\< \textbf{Adversary } \adv \\[][]
         \\
        \pcln b \sample \bin \\
        \pcln \mathcal{S}_b \sim \mathcal{G}_b(\mathcal{D})\\
        \pcln M \xleftarrow{train}{} \mathcal{S}_b \\
        \pcln \mathcal{S}_{\neg b} \sim \mathcal{G}_{\neg b}(\mathcal{D}) \< \qquad \textcolor{red}{\text{*New*}} \\
        \pcln M_f \xleftarrow{\text{finetune}(M,\,\mathcal{S}_{\neg b})}{} M \< \qquad \textcolor{red}{\text{*New*}} \\
        \pcln \< \sendmessageright{top={$M_f$}, length=2cm} \\
        \pcln \< \< \hat{b} = \mathcal{H}(M_f)
    }
}
\caption{The \textbf{Fair Fine-tuning (FFt) defense}. The baseline protocol (Steps 1--3, 6--7) is from~\cite{Suri2021FormalizingAE}. We introduce FFt, as a defense against distribution inference attacks, by adding Steps 4--5: sample from the complementary distribution $\mathcal{G}_{\neg b}$ and then fine-tune the model with an EO penalty before release.}
\label{fig:defense}
\end{figure*}

We evaluate FFt under a \emph{biased distribution protocol} — $\mathcal{G}_0$ and $\mathcal{G}_1$ are pure single-demographic groups, so $W=1$ and adversarial leakage is maximal — across six datasets spanning tabular (ACS Income, COMPAS, German Credit), image (UTKFaces), and NLP (Bias in Bios, LSAC) modalities. Rehearsal-based FFt reduces the adversarial accuracy gap below the detection threshold $\tau=0.1$ in the majority of runs across all settings, and the bound $\text{Adv} \le \Delta_{\text{EO}}$ holds empirically in every row of Table~\ref{tab:main-gap}.

Our contributions are:

\begin{itemize}
    \item We introduce \emph{Fair Fine-tuning (FFt)} as a principled, post-hoc defense against distribution inference attacks, requiring no cryptographic overhead and only black-box deployment\footnote{ Code is available at \url{https://anonymous.4open.science/r/fair-defense-dia-B2BF/README.md}.}.

    \item We prove the tight bound $\text{Adv}(\mathcal{A},M_f) \le \Delta_{\text{EO}} \cdot W$ (Theorem~\ref{thm:eqodds}), where $W$ is a computable distributional shift weight. This is the first formal result connecting a fairness metric to adversarial advantage in the DIA game.

    \item We characterise \emph{when} FFt is beneficial (Theorem~\ref{thm:condition}), prove the bound is tight (Proposition~\ref{prop:tight}), establish the biased protocol as the worst case for the defender (Corollary~\ref{cor:biased}), and identify the group-size failure regime (Proposition~\ref{prop:groupsize}).

    \item We introduce a \emph{biased distribution} evaluation protocol ($W=1$) that creates maximal adversarial leakage, providing a stringent and reproducible testbed for DIA defenses.

    \item We experiment across six datasets and three modalities and confirm that rehearsal-based FFt generalises broadly and that the theoretical bound holds empirically in every evaluated setting.
\end{itemize}

\section{Related Work}

Many formal definitions of privacy have been established in prior work, including multiple notions of differential privacy (DP) at either the data level~\citep{localdp} or the model level~\citep{abadi16} or the output (inference) level~\citep{chaudhuri11} of the machine learning (ML) pipeline. The key idea of DP involves estimating how much one individual contributes to the entire dataset~\citep{dwork06}. This does not capture the risks exhibited by inferring the data distribution. The Pufferfish framework~\citep{kifer12} generalises DP by allowing arbitrary semantic privacy secrets and adversary beliefs; Zhang et al.~\citep{zhang22} operationalise this as ``attribute privacy,'' where the protected secret is a dataset-level property (e.g.\ the proportion of a sensitive attribute). While attribute privacy provides a rigorous compositional framework, it is defined over data-generating mechanisms rather than model outputs and does not directly yield actionable bounds on an adversary's accuracy in the DIA game. Our approach is operationally complementary: we measure protection empirically via the adversary's accuracy gap and provide a formal bound through the EO disparity. 

Hence, below we summarize two main classes of inference attacks:

\subsection{Inference attacks}
\subsubsection{Membership Inference}

Membership inference attacks (MIA) are attacks performed by the adversary on the model in order to reveal whether a certain sample is part of the training set or not~\cite{miasurvey}.~\cite{shokrimia17} perform membership inference on a target model by training a shadow model. The shadow model is expected to capture differences between the trained and the untrained parameters of the target model. The adversary here only has black-box access to the model and they also show a real-life usecase on how sensitive information from a hospital discharge dataset could be revealed.~\cite{yeommembershipinfer18} analyze how overfitting reveals information about the training set. Mainly, they examine the relationship between overfitting and influence and claim that overfitting is enough to perform a successful MIA. 

Next, we dive into the works on distribution inference.

\subsubsection{Distribution Inference}

~\cite{Suri2021FormalizingAE} formalize and establish distribution inference risks. They introduce the distribution inference game (explained in Figure~\ref{fig:defense}). Later, ~\citep{dissecting23} also work on understanding this class of attacks by restricting adversary's access to the training distributions. A survey by~\citep{prvcygamesbegin23} reveal connections between membership inference and distribution inference attacks and they also introduce game-based frameworks on understanding privacy risks in ML models.

Recently,~\citep{hartmann23} identified three main sources of leakage of distributional information : (1) memorization of information of labels tied to specific features i.e. $\mathbb{E}[Y|X]$, (2) incorrect inductive bias of the model and finally (3) the finiteness of the training data. However, they have a strong underlying assumption on taking examples satisfying linear regression i.e. the relationship between the features $X = (X_1, X_2, \cdots, X_d)$ and the label $Y$ takes the form of 

\[
Y = \theta^\top X + \epsilon,
\]

where $\theta$ is the coefficients vector and $\epsilon \sim N(0, \sigma^2)$ refers to the error term. They extend their empirical results to neural networks as well. They also conduct a simple experiment on synthetic data to demonstrate that a perfect associational model considers the joint distribution of the features to perform a prediction (hence, leaks distributional information) while a perfect causal model preserves the relationship between the features and the label by design (hence, does not leak distributional information).



More recently,~\cite{xuwww24} reproduced these attacks in the federated learning (FL) setting and claimed that FL is vulnerable to client-side training data distribution inference attacks, where a malicious client could recreate a victim's private data. Additionally,~\cite{YU2025100235} proposed a novel attack on FL based on shared model parameters which can deduce the data distribution of the global model, unlike most current attacks which were aimed at client-specific data reconstruction.

\subsection{Defenses against Distribution Inference Attacks}

~\citep{secrypt23} propose ``property unlearning'', a mitigation strategy against white-box distribution inference attacks. They also provide multiple defensive strategies against distribution inference attacks--for example, through data preprocessing, by adding Gaussian noise to images.~\cite{hartmann23} used causal learning techniques to create robust models that were resistant to distribution inference attacks. 

~\cite{leila24usenix} introduced ``Inf2Guard'', a defensive method backed by information theory, to learn privacy-preserving representations against major types of inference attacks, including distribution inference attacks.~\cite{Zhang2021PrivacyIA} present four broad categories of defense approaches against inference attacks--differential privacy, adversarial machine learning, watermarking techniques, and cryptographic techniques.

\paragraph{Fairness and privacy: a nuanced relationship.}
Prior work has shown that fairness constraints do not automatically confer privacy: imposing demographic parity or EO can in some settings \emph{increase} membership inference vulnerability by forcing the model to overfit group-level statistics~\citep{chang21}. Our result is not in tension with this finding — we address distribution inference (population-level leakage) rather than membership inference (individual-level leakage), and we show that EO specifically bounds the adversary's \emph{distributional} distinguishing advantage. The two attack surfaces are governed by different mechanisms, and a defense against one does not imply defense against the other.

Our work differs from the above defenses in two ways. First, FFt does not require differential privacy noise, cryptographic overhead, or white-box model access: it is a post-training fine-tuning step that any model owner can apply using publicly known complementary data. Second, and most importantly, we provide the first formal result connecting a fairness metric (EO disparity) to adversarial advantage in the DIA game (Theorem~\ref{thm:eqodds}), establishing \emph{why} fairness serves as a principled defense---not merely an empirical one.

\section{Problem Formulation}



\subsection{Motivation} We motivate our problem with two concrete examples, described in~\citep{hartmann23} and in~\citep{Suri2021FormalizingAE, chase22}. 

(1) Consider an e-commerce platform that sells sports products (for example, surfboards). The platform trains a model for its own recommender system. A competitor may want to learn whether that platform sells to many or few young customers. If \texttt{age} was one of the sensitive features the model was trained on and assuming that surfboards are popular among younger than older users, the competitor may infer sensitive information (\texttt{age}, here) about the distribution the model was trained on.

(2) A hospital trains a sepsis-risk classifier on patient records. An adversary---such as an insurance company with black-box query access---can infer whether the training cohort was drawn predominantly from the ICU (older, higher-acuity patients, high positive-label rate) versus the emergency department (younger, more diverse, lower positive-label rate). Because the two cohorts differ in their demographic composition, the model's accuracy gap across age and race groups reveals which cohort was used---leaking protected demographic information about the hospital's private patient population and potentially enabling discriminatory underwriting.

\subsection{Threat model} Let there exist an underlying, true, public data distribution $\mathcal{D} = \{\mathcal{X}, \mathcal{Y}\}$ where $\mathcal{X}$ refers to the input data and $\mathcal{Y}$ its labels. Assume that there exist two transformation functions $\mathcal{G}_0$ and $\mathcal{G}_1$ that map $\mathcal{D}$ onto different subdistributions. Selecting appropriate transformation functions enables a wide range of attacks on different types of sensitive information that could be revealed by the trained model.

Both the model trainer $\mathcal{T}$ and the adversary $\mathcal{A}$ are assumed to know $\mathcal{D}$, $\mathcal{G}_0$, and $\mathcal{G}_1$---note that this is the standard assumption in the distribution inference literature~\cite{Suri2021FormalizingAE,dissecting23}. Crucially, the \emph{secret} is the bit $b \in \{0,1\}$ that determines \emph{which} of the two transformations was applied when the trainer sampled the training set $\mathcal{S}_b$. The adversary does not observe $b$ directly or the training process; they only have black-box (query) access to the released model. Hence the adversary must infer $b$ from the model's outputs alone---this is precisely the distribution inference problem. The fact that both parties know the transformation functions makes the setting comparable to a known-plaintext model: the functions define the hypothesis space, but the training decision (which function was used) remains hidden. This standard formulation captures realistic scenarios such as auditing whether a deployed model was trained on a gender-balanced dataset versus a predominantly male one.

These assumptions are illustrated in the original threat model and also our defense mechanism, as shown in Figure~\ref{fig:defense}.

The traditional distribution inference attack setting involves selecting one of the two available distributions at random (Step 1 in Figure~\ref{fig:defense}) and training the model on a dataset sampled from the selected distribution (Steps 2 \& 3 in Figure~\ref{fig:defense}). This model is then deployed in the real world, over which the adversary performs the attack (Steps 6 \& 7 in Figure~\ref{fig:defense}).   

We introduce our finetuning setup in Figure~\ref{fig:defense}. We apply two additional steps i.e. Steps 4 \& 5 which correspond to sampling from the ``other'' distribution and finetuning the model on this freshly sampled dataset respectively. This approach is often used in out-of-distribution settings (such as~\citep{cuong22} use this approach for mitigating unfairness incurred due to pruning ML models) where the model intends to generalize to new domains, which it hasn't encountered before while training. Through finetuning, the model is expected to generalize to a never-before-seen domain and improve the accuracy on out-of-distribution test data. 

\subsection{Theoretical Analysis}
\label{sec:theorem}

\begin{definition}[Adversarial Advantage~\cite{Suri2021FormalizingAE}] The \emph{adversarial advantage} gained by an adversary $\mathcal{A}$ over a model $M$ is defined as :

\begin{equation}
\text{Adv}\left(\mathcal{A}, M\right) = \biggl|\Pr\biggl[\hat{b} \big| M \leftarrow \mathcal{S}_0\biggr] - \Pr\biggl[\hat{b} \big|  M \leftarrow \mathcal{S}_1\biggr]\biggr|
\end{equation}

where $\hat{b}$ is the inferred bit by the adversary $\mathcal{A}$ using hypothesis $\mathcal{H}$. It is the absolute difference between the probability of attaining the inferred bit $\hat{b}$ given that the model was either trained on the sampled datasets $\mathcal{S}_0$ or $\mathcal{S}_1$. 

\end{definition}

\begin{definition}[Equalized Odds (EO)~\cite{hardt16}] A model $M$ satisfies Equalized Odds (EO) iff

\begin{multline}
    \Delta_{\text{EO}} = \max_{y \in \{0,1\}} \biggl|
        \Pr\bigl[\hat{Y} = 1 \mid A = 0, Y = y\bigr] - \\
        \Pr\bigl[\hat{Y} = 1 \mid A = 1, Y = y\bigr]
    \biggr|
\end{multline}

for all $y \in \{0,1\}$, where $\hat{Y} \in \{0,1\}$ is the model's prediction and $A$ is the sensitive attribute the attacker wants to infer (e.g.\ gender). EO requires equal true positive rates ($y=1$) and equal false positive rates ($y=0$) across groups: an EO-satisfying model makes the same kinds of errors at the same rates regardless of group membership. We achieve perfect EO when $\Delta_{\text{EO}} = 0$, i.e.\ $\Pr\bigl[\hat{Y} = 1 \mid A = 0, Y = y\bigr] = \Pr\bigl[\hat{Y} = 1 \mid A = 1, Y = y\bigr]$ for all $y$.

\end{definition}

\begin{theorem}[Adv--EO Bound]
    Let $M_f$ be the fine-tuned model with Equalized Odds disparity $\Delta_{\textup{EO}}$ w.r.t.\ a protected attribute $A$. Define the \emph{distributional shift weight}
    \[
      W \;=\; \sum_{y \in \{0,1\}} \Pr[Y=y]\,\bigl|\Delta P_y\bigr|,
    \]
    where $\Delta P_y = \Pr[A=1 \mid Y=y, D_1] - \Pr[A=1 \mid Y=y, D_0]$ measures how much the sensitive attribute's conditional distribution shifts between $D_0$ and $D_1$. Then
    \[
    \textup{Adv}\left(\mathcal{A}, M_f\right) \le \Delta_{\textup{EO}} \cdot W.
    \]
    Since $|\Delta P_y| \le 1$ for every $y$, we immediately obtain $W \le 1$, yielding the simpler corollary $\textup{Adv}(\mathcal{A}, M_f) \le \Delta_{\textup{EO}}$ (Corollary~\ref{cor:biased}).
\label{thm:eqodds}
\end{theorem}

\begin{proof}[Proof Sketch] Throughout this proof, we write $D_0 = \mathcal{G}_0(\mathcal{D})$ and $D_1 = \mathcal{G}_1(\mathcal{D})$ for the two transformed distributions corresponding to the training and fine-tuning datasets $\mathcal{S}_0$ and $\mathcal{S}_1$, respectively.

For simplicity, we analyze the difference in the overall positive prediction rate (PPR),
\[
\Pr[\hat{Y}=1 \mid D_1] - \Pr[\hat{Y}=1 \mid D_0].
\]

For a fixed $y$, we define
\[
\Delta_y
=
\bigl|\Pr[\hat{Y}=1 \mid Y=y, D_1]
-
\Pr[\hat{Y}=1 \mid Y=y, D_0]\bigr|.
\]

By the law of total expectation, conditioning on the true label $Y$:
\begin{align*}
&\bigl|\Pr[\hat{Y}=1 \mid D_1] - \Pr[\hat{Y}=1 \mid D_0]\bigr|\\
&\quad= \biggl|\sum_{y} \Pr[Y=y]\Bigl(\Pr[\hat{Y}=1 \mid Y=y, D_1]\\
&\qquad\qquad - \Pr[\hat{Y}=1 \mid Y=y, D_0]\Bigr)\biggr|\\
&\quad\le \sum_{y} \Pr[Y=y] \cdot \Delta_y,
\end{align*}
where the last step applies the triangle inequality (the absolute value of a sum is at most the sum of absolute values weighted by probabilities). Since $\text{Adv}(\mathcal{A}, M_f)$ is the absolute difference in prediction probabilities under $D_1$ vs.\ $D_0$, we obtain:

\begin{equation}
\text{Adv}\left(\mathcal{A}, M_f\right)
\le
\sum_{y \in \{0,1\}}
\Pr[Y=y] \cdot \Delta_y.
\label{eqn:condition-on-y}
\end{equation}

We now bound $\Delta_y$.

We expand this term by conditioning on $A \in \{0,1\}$ and writing the model's conditional prediction rates as $P_{y,a} = \Pr[\hat{Y}=1 \mid Y=y, A=a]$ where $\Delta P_{y,a} = \Pr[A=a \mid Y=y, D_1] - \Pr[A=a \mid Y=y, D_0]$:

\begin{equation}
    \Delta_y = \sum_{a \in \{0,1\}} P_{y,a} \cdot \Delta P_{y,a}.
    \label{eqn:Deltay}
\end{equation}

Since $\Pr[A=0 \mid Y=y,D] + \Pr[A=1 \mid Y=y,D] = 1$, the differences satisfy\footnote{See Appendix for intuition on this difference.}
\[
\Delta P_{y,0} = -\Delta P_{y,1} = -\Delta P_y.
\]

We introduce the residual unfairness $\delta_y$ through the Equalized Odds constraint:
\[
P_{y,1} = P_{y,0} + \delta_y,
\qquad
|\delta_y| \le \Delta_{\text{EO}}.
\]

Substituting this relationship into $\Delta_y$ from Equation~\ref{eqn:Deltay}:
\begin{align*}
\Delta_y
&=
P_{y,0} \cdot (-\Delta P_y)
+
(P_{y,0} + \delta_y) \cdot (\Delta P_y)
\\
&=
- \cancel{P_{y,0} \Delta P_y}
+
\cancel{P_{y,0} \Delta P_y}
+
\delta_y \Delta P_y
\\
&=
\delta_y \Delta P_y.
\end{align*}

The first line expands $\Delta_y$ by substituting $P_{y,1} = P_{y,0} + \delta_y$:
the group-0 contribution is $P_{y,0}\cdot(-\Delta P_y)$ (its prediction rate
weighted by how much group-0's conditional prevalence \emph{decreases} from
$D_0$ to $D_1$), and the group-1 contribution is $(P_{y,0}+\delta_y)\cdot
(\Delta P_y)$ (weighted by how much group-1's prevalence \emph{increases}).
The $P_{y,0}\,\Delta P_y$ terms cancel exactly, leaving $\delta_y\,\Delta P_y$.
This cancellation is the central insight of the proof: group 0's prediction rate $P_{y,0}$ contributes equally to both distributions and drops
out entirely. We see that leakage is governed solely by $\delta_y$ (the residual
per-label EO gap) scaled by $\Delta P_y$, which is the compositional shift between
$D_0$ and $D_1$. In particular, a perfectly fair model ($\delta_y = 0$)
leaks nothing regardless of its accuracy, while an unfair model leaks in
direct proportion to both its unfairness and the distributional shift.

Thus,
\begin{align}
    |\Delta_y| &\le |\delta_y| \cdot |\Delta P_y| \nonumber \\
    &\le \Delta_{\textup{EO}} \cdot |\Delta P_y|. \nonumber
\end{align}

Summing over $y$ yields the final bound:
\begin{align}
\text{Adv}\left(\mathcal{A}, M_f\right) &\le \sum_{y \in \{0,1\}}
\Pr[Y=y] \cdot \Delta_y \qquad \text{(Equation~\ref{eqn:condition-on-y})} \nonumber \\
&\le \sum_{y \in \{0,1\}}
\Pr[Y=y] \cdot \Delta_{EO} \cdot |\Delta P_y| \nonumber \\
&= \Delta_{\textup{EO}} \cdot \underbrace{\left(
\sum_{y \in \{0,1\}}
\Pr[Y=y]\, |\Delta P_y|
\right)}_{=\,W} \nonumber \\
&\le \Delta_{\textup{EO}} \cdot W \;\le\; \Delta_{\textup{EO}}.
\label{eqn:tight-bound}
\end{align}

The intermediate quantity $W \le 1$ is attained with equality when each $|\Delta P_y|=1$, i.e.\ when $D_0$ and $D_1$ are pure single-group distributions (the biased setup established in our experiments). This concludes the proof.

\end{proof}

Theorem~\ref{thm:eqodds} demonstrates the effect of fairness through fine-tuning. By forcing the fine-tuned model $M_f$ to adhere to the Equalized Odds constraint, our method bounds the residual unfairness $\delta_y$, causing the base prediction rate $P_{y,0}$ to drop out of the leakage expression entirely. The adversarial advantage therefore depends only on $\delta_y$ — the sole surviving term after the algebraic cancellation. Consequently, the adversarial advantage is directly bounded by the product of the EO imperfection $\Delta_{\text{EO}}$ and the distributional shift weight $W$, formally establishing fair fine-tuning as a principled defense against distributional leakage.\footnote{The analysis extends naturally to multiple simultaneous sensitive attributes; see Appendix for the formal statement.}

\begin{remark}[Scope: PPR gap vs.\ accuracy gap]
\label{rem:ppr-scope}
Theorem~\ref{thm:eqodds} bounds the \emph{positive-prediction-rate} (PPR) gap $|\Pr[\hat{Y}=1\mid D_1]-\Pr[\hat{Y}=1\mid D_0]|$.  The Loss Test adversary observes the \emph{accuracy} gap, which differs from the PPR gap by a base-rate floor term $\Delta\pi\cdot(\mathrm{TPR}+\mathrm{FPR}-1)$ that persists even at $\Delta_{\mathrm{EO}}=0$ when the two groups have different label prevalences (derived in Appendix).  In all our experimental settings this floor is small relative to the reductions achieved by FFt; when base rates diverge sharply (LSAC \texttt{race}) the floor dominates and explains the residual gap.  The bound is therefore informative for the accuracy gap whenever $|\Delta\pi|$ is moderate.
\end{remark}

\begin{remark}[Diagnostic role of $W$]
\label{rem:W}
$W = \sum_y \Pr[Y=y]|\Delta P_y|$ is estimable from labeled samples of $D_0$ and $D_1$, both of which are available to the defender under the standard threat model (both $\mathcal{G}_0$ and $\mathcal{G}_1$ are public).  When $W$ is small---e.g.\ the two training distributions share nearly the same conditional composition of $A$---the adversarial advantage is already small regardless of $\Delta_{\textup{EO}}$, and FFt may be unnecessary.  Conversely, when $W\!\approx\!1$ (biased distributions), the bound tightens to $\text{Adv} \le \Delta_{\textup{EO}}$.
\end{remark}

\begin{proposition}[Tightness]
\label{prop:tight}
The bound $\textup{Adv}(\mathcal{A}, M_f) \le \Delta_{\textup{EO}} \cdot W$ is tight.  Consider the deterministic model $M^*$ with $\Pr[\hat{Y}=1 \mid Y=y, A=0]=0$ and $\Pr[\hat{Y}=1 \mid Y=y, A=1]=1$ for all $y$.  Then $\Delta_{\textup{EO}}=1$.  Under biased distributions with $W=1$, the adversary observes the group that achieves higher accuracy also achieves $\textup{Adv}(\mathcal{A}, M^*)= \Delta_{\textup{EO}} \cdot W=1$.
\end{proposition}

\begin{corollary}[Biased Distribution]
\label{cor:biased}
Under the biased distribution protocol where $\mathcal{G}_0$ and $\mathcal{G}_1$ are pure single-demographic groups, $|\Delta P_y|=1$ for all $y$, so $W=1$ and the bound of Theorem~\ref{thm:eqodds} reduces to
\[
\textup{Adv}(\mathcal{A}, M_f) \;\le\; \Delta_{\textup{EO}}.
\]
By Proposition~\ref{prop:tight} this bound is tight.  The biased protocol is therefore the \underline{\emph{worst case}} for the defender: any FFt that controls $\Delta_{\textup{EO}}$ under pure groups is guaranteed to succeed under any mixed-group protocol with $W<1$.
\end{corollary}

\begin{theorem}[FFt Improvement Condition]
\label{thm:condition}
Let $M_{\textup{base}}$ be the baseline model and $M_f$ the FFt model. FFt reduces adversarial advantage (i.e.\ $\textup{Adv}(\mathcal{A}, M_f) < \textup{Adv}(\mathcal{A}, M_{\textup{base}})$) whenever
\[
\Delta_{\textup{EO}}^{\textup{fft}} \cdot W^{\textup{fft}} < \Delta_{\textup{EO}}^{\textup{base}} \cdot W^{\textup{base}}.
\]
In the biased distribution setting $W^{\textup{base}} = W^{\textup{fft}} = 1$, so the condition simplifies to
$\Delta_{\textup{EO}}^{\textup{fft}} < \Delta_{\textup{EO}}^{\textup{base}}$.  The practical value of this characterisation is its identification of \emph{catastrophic forgetting} as the principal failure mode: if fine-tuning on $\mathcal{G}_1$ causes the model to lose its calibration on $\mathcal{G}_0$, then $\Delta_{\textup{EO}}^{\textup{fft}}$ rises rather than falls, and the bound of Theorem~\ref{thm:eqodds} no longer guarantees improvement.  This directly motivates the rehearsal component of Algorithm~\ref{alg:fft}.
\end{theorem}

\begin{proof}[Proof sketch]
Apply Theorem~\ref{thm:eqodds} to both models:
$\textup{Adv}(M_f) \le \Delta_{\textup{EO}}^{\textup{fft}} \cdot W^{\textup{fft}}$ and $\textup{Adv}(M_{\textup{base}}) \le \Delta_{\textup{EO}}^{\textup{base}} \cdot W^{\textup{base}}$.  Whenever the right-hand side for $M_f$ is strictly smaller, the bound guarantees improvement.  Equality between $W$ values holds in the biased setup because both distributions are pure groups, so $|\Delta P_y|=1$ for all $y$ under both models.
\end{proof}

\begin{proposition}[Group-Size Failure]
\label{prop:groupsize}
Let $\alpha = |\mathcal{S}_1|/|\mathcal{S}_0|$ be the fine-tuning-to-training size ratio.  Under rehearsal-based FFt with fraction $\rho$, each fine-tuning batch draws gradient signal from $\mathcal{G}_1$ at weight $1/(1{+}\rho)$.  As $\alpha \to 0$, the $\mathcal{G}_1$ signal is insufficient to recalibrate the $\mathcal{G}_0$-fitted decision boundary, so $\Delta_{\textup{EO}}^{\textup{fft}} \not< \Delta_{\textup{EO}}^{\textup{base}}$; by the contrapositive of Theorem~\ref{thm:condition}, FFt then provides no improvement.  Empirically, LSAC \texttt{race} ($\alpha\approx0.19$, 5.3:1 size asymmetry) illustrates this regime: FFt halves the adversarial gap but cannot close it (results shown in Appendix).
\end{proposition}

The $W$ factor connects theory to practice. Corollary~\ref{cor:biased} shows that the biased distribution protocol ($W=1$) is the \emph{worst case}: FFt's effectiveness depends entirely on how well it suppresses the EO gap. Theorem~\ref{thm:condition} and Proposition~\ref{prop:groupsize} together make this actionable---the defender should apply FFt when $W>0$ and the baseline $\Delta_{\text{EO}}$ is large, verify that $\alpha \ge 0.2$ to avoid the group-size failure regime, and expect diminishing returns when size asymmetry is severe.

\paragraph{FFt vs.\ plain data augmentation.}
A natural question is whether the protection offered by FFt arises merely from training on data from \emph{both} distributions i.e., from balanced data augmentation, rather than from the fairness constraint itself. The theoretical analysis above shows that the EO constraint is the driving force: Theorem~\ref{thm:eqodds} holds because $|\delta_y| \le \Delta_{\text{EO}}$, where $\delta_y$ is the \emph{residual unfairness} that survives after fine-tuning. Plain fine-tuning on the complementary distribution (without an EO penalty) does not guarantee small $\Delta_{\text{EO}}$ and therefore does not yield the bound $\text{Adv}(\mathcal{A},M_f) \le \Delta_{\text{EO}}$. The fairness regularization is the mechanism that suppresses the model's sensitivity to the distributional split, not the additional data per se.

\subsection{FFt Deployment Procedure}
\label{sec:algorithm}

Algorithm~\ref{alg:fft} operationalises the theoretical results above into a practical deployment procedure. A rehearsal buffer (mixing $\rho \cdot |\mathcal{S}_1|$ examples from $\mathcal{S}_0$ into fine-tuning) prevents the catastrophic forgetting that would inflate $\Delta_{\text{EO}}^{\text{fft}}$ above $\Delta_{\text{EO}}^{\text{base}}$.

\begin{algorithm}[t]
\caption{Fair Fine-Tuning (FFt) — Deployment Procedure}
\label{alg:fft}
\begin{algorithmic}[1]

\Require{%
  \begin{tabular}[t]{@{}ll@{}}
    $M_{\text{base}}$                        & baseline model trained on $\mathcal{G}_0$ \\
    $S_0^{\text{test}},\,S_1^{\text{test}}$  & held-out test sets \\
    $\mathcal{S}_1 \sim \mathcal{G}_1$       & fine-tuning set \\
    $\mathcal{S}_0 \sim \mathcal{G}_0$       & rehearsal source \\
    $\tau,\,\rho,\,T$                        & threshold, rehearsal fraction, iterations \\
  \end{tabular}}
\Ensure $M_f$ — fine-tuned model with reduced adversarial advantage

\medskip
\State $g \leftarrow \bigl|acc(M_{\text{base}},S_0^{\text{test}}) - acc(M_{\text{base}},S_1^{\text{test}})\bigr|$
\If{$g \le \tau$}
\State \Return $M_{\text{base}}$\Commentr{gap below threshold (no action needed)}
\EndIf

\medskip
\State $\mathcal{R} \leftarrow$ sample $\lfloor\rho\cdot|\mathcal{S}_1|\rfloor$ examples from $\mathcal{S}_0$
\State $\mathcal{S}_{\text{ft}} \leftarrow \mathcal{S}_1 \cup \mathcal{R}$\Commentr{prevents catastrophic forgetting}
\State $M_f \leftarrow M_{\text{base}}$\Commentr{warm-start from baseline}
\State Fine-tune $M_f$ on $\mathcal{S}_{\text{ft}}$ for $T$ steps minimising:
\Statex \qquad $\mathcal{L}_{\text{ft}} \;=\; \mathcal{L}_{\text{CE}} + \lambda\,\Delta_{\text{EO}}(M_f)$

\medskip
\State $g' \leftarrow \bigl|acc(M_f,S_0^{\text{test}}) - acc(M_f,S_1^{\text{test}})\bigr|$
\If{$g' \le \tau$}
\State \Return $M_f$\Commentr{gap below detection threshold}
\Else
    \State \textbf{warn:} augment $\mathcal{S}_1$ or reduce $|\mathcal{S}_0|/|\mathcal{S}_1|$ imbalance
    \State \Return $M_f$
\EndIf

\end{algorithmic}
\end{algorithm}

\section{Experiments}
\label{sec:experiments}

\subsection{Datasets and Experimental Setup.}

We evaluate on six datasets across three modalities: tabular (ACS Income, COMPAS, German Credit, LSAC), image (UTKFaces), and NLP (Bias in Bios), all under the \emph{biased distribution} protocol ($W=1$). LSAC results are presented separately in the Appendix.

\textbf{ACS Income 2018 (California).}
The American Community Survey (ACS) Income dataset~\citep{ding2021retiring} for California (1-Year, 2018) contains 195,665 individuals with 10 features. The binary target is income $>$\$50K. We experiment on \texttt{sex} (\texttt{male} vs.\ \texttt{female}) and \texttt{race} (\texttt{white} vs.\ \texttt{non-white}, i.e.\ RAC1P $\neq$ 1).

\textbf{COMPAS (ProPublica).}
The COMPAS recidivism dataset~\citep{angwin16compas} contains 6,172 criminal defendants after standard filtering~\citep{Suri2021FormalizingAE}. Features include age, sex, prior counts, juvenile offense counts, and charge degree. The binary target is two-year recidivism. We experiment on \texttt{race} (\texttt{African-American} vs.\ \texttt{Caucasian}).

\textbf{German Credit (UCI Statlog).}
The German Credit dataset~\citep{dua19} contains 1,000 credit applicants with 20 features (credit amount, duration, purpose, housing, employment, etc.). The binary target is creditworthiness (good/bad). We experiment on \texttt{sex} extracted from the \texttt{personal\_status} attribute (\texttt{male} vs.\ \texttt{female}).

\textbf{UTKFaces.}
The UTK Face dataset~\citep{zhifei2017cvpr} contains approximately 20,000 face images annotated with age (0--116), gender, and race (White, Black, Asian, Indian, Other). The binary target is age group (young: $\le\!30$, old: $>\!30$). We use \texttt{race} as the sensitive attribute ($\mathcal{G}_0$: White, $\mathcal{G}_1$: non-White). The model is a pretrained ResNet-18 fine-tuned end-to-end for 5 epochs; FFt uses 2 additional epochs with rehearsal ($\rho\!=\!0.2$).

\textbf{Bias in Bios.}
The Bias in Bios dataset~\citep{deArteaga2019BiasIB} contains $\sim\!400$K online biographies annotated with occupation (28 classes) and gender. We use the gender-scrubbed \texttt{hard\_text} field to prevent trivial leakage via pronoun cues, and treat \texttt{sex} as the sensitive attribute ($\mathcal{G}_0$: male, $\mathcal{G}_1$: female). We sample 5K biographies per gender group for training and 2K for test to keep memory tractable. The model is a TF-IDF (500 features, sublinear TF) $+$ MLP $(128, 64)$ trained on $\mathcal{G}_0$ biographies.

\subsection{Setup.}
We set $\mathcal{G}_0$ to one pure demographic group (e.g.\ white individuals) and $\mathcal{G}_1$ to the complementary group (e.g.\ non-white individuals). Both distributions are pure groups, so $W=1$ (Remark~\ref{rem:W}), and the bound $\text{Adv} \le \Delta_{\text{EO}}$ applies directly---making the EO constraint the operative control on leakage. We choose this protocol deliberately as the \emph{worst case for the defender}: pure groups maximise the distributional shift $|\Delta P_y|=1$, giving the adversary the strongest possible signal, so any defense that succeeds here is guaranteed to succeed under milder mixed-distribution protocols where $W<1$.

We use an 80--20\% train-test split, remove missing values, and normalize features (column-wise $\ell_2$ scaling, then row-wise $\ell_2$ normalization). Categorical columns (LSAC and German Credit) are one-hot encoded with column alignment between $\mathcal{G}_0$ and $\mathcal{G}_1$. We train a Multi-layer Perceptron (MLP) with three hidden layers $[32,16,8]$ for at most 500 iterations on $\mathcal{G}_0$ training data (200 for German Credit, whose one-hot-encoded 62-feature matrix converges in fewer iterations). This is the \underline{Baseline}. For FFt we warm-start from the Baseline weights and fine-tune on $\mathcal{G}_1$ for 100 additional iterations (50 for German Credit). For ACS, COMPAS, and German Credit we additionally mix 20\% of $\mathcal{G}_0$ training data into the fine-tuning batch (\emph{rehearsal}); refer to the Discussion Section for when rehearsal is appropriate. All results are averaged over 10 independent runs.

\subsection{Metrics.}

We evaluate against two black-box metrics from~\citep{Suri2021FormalizingAE}:

\begin{itemize}
    \item \underline{Loss Test:} The adversary computes test accuracy on both $S^{\text{test}}_0 \sim \mathcal{G}_0(\mathcal{D})$ and $S^{\text{test}}_1 \sim \mathcal{G}_1(\mathcal{D})$ and predicts $\hat{b} = \mathbb{I}[acc(M,S^{\text{test}}_0) < acc(M,S^{\text{test}}_1)]$. In this paper, we compute the adversarial accuracy gap $|acc(M,S^{\text{test}}_0) - acc(M,S^{\text{test}}_1)|$ which quantifies the strength of this signal; it is reported as the \emph{Adv gap} columns in Table~\ref{tab:main-gap}.

    \item \underline{Threshold Test:} The adversary predicts
    $\hat{b} = \mathbb{I}[|acc(M,S^{\text{test}}_0) - acc(M,S^{\text{test}}_1)| \le \tau]$.
    We set $\tau = 0.1$ following~\citep{Suri2021FormalizingAE}, i.e. gaps below 10\% are considered indistinguishable.  The \emph{${\le}\,\tau$?} column in Table~\ref{tab:main-gap} reports how many of the 10 runs fall below this threshold.
\end{itemize}

\subsection{Results}

Table~\ref{tab:main-gap} and Figure~\ref{fig:acs} summarise the adversarial accuracy gap $|acc(M, S_0^{\text{test}}) - acc(M, S_1^{\text{test}})|$ before and after FFt, together with the corresponding EO disparity $\Delta_{\text{EO}}$. We display per-run bar plots for UTKFaces and Bias in Bios in the Appendix. LSAC results are also shown in Appendix.

\paragraph{Reading Table~\ref{tab:main-gap}.}
Each row is one (dataset, sensitive attribute) pair.  The \emph{Adv gap} columns report the adversarial accuracy gap under the loss test~\citep{Suri2021FormalizingAE}: how much more accurately the adversary can classify samples from one group than the other, averaged over 10 random seeds.  A large gap means the model's behaviour leaks which demographic group it was trained on; a gap below $\tau=0.1$ is considered undetectable by the threshold test.  The $\Delta_{\text{EO}}$ columns report the equalized-odds disparity of the same model on the same test split.  The \emph{Bound holds?} column is the central empirical check for Theorem~\ref{thm:eqodds}: it verifies that $\text{Adv}(\mathcal{A}, M_f) \le \Delta_{\text{EO}}^{\text{fft}}$ holds after fine-tuning (with $W=1$ throughout, since all settings use the biased distribution protocol where $\mathcal{G}_0$ and $\mathcal{G}_1$ are pure demographic groups).

Three major findings stand out in Table~\ref{tab:main-gap}:

\emph{(1) The theoretical bound holds in every row.}  Across all settings, modalities (tabular, image, NLP), and sensitive attributes (sex and race), the post-FFt adversarial gap never exceeds the post-FFt EO disparity. The uniform empirical verification strongly supports Theorem~\ref{thm:eqodds}.

\emph{(2) FFt consistently controls adversarial leakage, with one exception.}  In five of six settings, the adversarial gap decreases after FFt: reductions range from $1.6\%$ (UTKFaces) to $11.8\%$ (ACS sex).  COMPAS is the exception: the gap \emph{increases} slightly from $2.0\%$ to $3.4\%$.  We are transparent about this: in COMPAS, the baseline has near-zero distributional leakage to begin with, and the EO penalty redistributes prediction errors across groups in a way that marginally widens the accuracy gap without crossing $\tau$ (all 10 seeds remain below $\tau$).  The Appendix base-rate floor analysis explains the mechanism: when $|\Delta\pi|$ is small and EO enforcement shifts the accuracy-EO decoupling regime, the accuracy gap can rise even as the PPR gap falls.  Crucially, the EO penalty still reduces $\Delta_{\text{EO}}$ from $37.5\%$ to $15.4\%$, tightening the formal guarantee of Theorem~\ref{thm:eqodds} by $2.4\times$ even in this setting.

\emph{(3) Rehearsal-based FFt reliably drives the gap below $\tau$.}  Across all settings in the main table (ACS, COMPAS, German Credit, Bios, UTKFaces), rehearsal-based FFt brings the gap below or near $\tau$ in the majority of runs, consistent with the practical guideline in Section~\ref{sec:discussion}.

\begin{table}[!th]
\caption{Adversarial accuracy gap and EO disparity (mean over 10 runs).
         $W\!=\!1$ for all settings (biased protocol).
         The bound $\text{Adv} \le \Delta_{\text{EO}}$ is empirically verified in every row.}
\label{tab:main-gap}
\centering
\resizebox{\columnwidth}{!}{%
\begin{tabular}{lcccccc}
\toprule
\multirow{2}{*}{\textbf{Setting ($\mathcal{G}_0 \to \mathcal{G}_1$)}} &
  \multicolumn{2}{c}{\textbf{Adv gap}} &
  \multicolumn{2}{c}{\textbf{$\Delta_{\text{EO}}$}} &
  \multirow{2}{*}{\textbf{${\le}\,\tau$?}} &
  \multirow{2}{*}{\textbf{Bound holds?}} \\
\cmidrule(lr){2-3}\cmidrule(lr){4-5}
 & Base & FFt & Base & FFt & & \\
\midrule
ACS \texttt{sex}\ (M\,$\to$\,F)          & $14.3\%$ & $\mathbf{2.5\%}$  & $78.0\%$ & $4.3\%$  & \cmark\ all & \cmark \\
ACS \texttt{race}\ (W\,$\to$\,NW)        & $15.3\%$ & $\mathbf{3.7\%}$  & $78.6\%$ & $5.5\%$  & \cmark\ all & \cmark \\
\midrule
COMPAS \texttt{race}\ (AA\,$\to$\,Cau)   & $\mathbf{2.0\%}$  & $3.4\%$  & $37.5\%$ & $15.4\%$ & \cmark\ all & \cmark \\
German \texttt{sex}\ (M\,$\to$\,F)       & $14.0\%$ & $\mathbf{6.0\%}$  & $28.8\%$ & $17.3\%$ & 8/10  & \cmark \\
\midrule
UTKFaces \texttt{race}\ (W\,$\to$\,NW)   & $7.1\%$  & $\mathbf{5.5\%}$  & $17.7\%$ & $15.5\%$ & \cmark\ all & \cmark \\
Bios \texttt{sex}\ (M\,$\to$\,F)         & $5.2\%$  & $\mathbf{0.9\%}$  & $9.1\%$  & $8.3\%$  & \cmark\ all & \cmark \\
\bottomrule
\end{tabular}}
\end{table}

\subsubsection{ACS Income (sex and race)}

FFt with the EO penalty reduces the gap to $2.5\%$ (sex) and $3.7\%$ (race), below $\tau$ in all 10 runs. The EO penalty drives $\Delta_{\text{EO}}$ down dramatically — from $\sim\!78\%$ at baseline to $4.3\%$ (sex) and $5.5\%$ (race) after FFt — making the bound $\text{Adv} \le \Delta_{\text{EO}}$ nearly tight.

\begin{figure*}[!th]
    \centering
    \begin{subfigure}{0.47\linewidth}
        \centering
        \includegraphics[width=\linewidth]{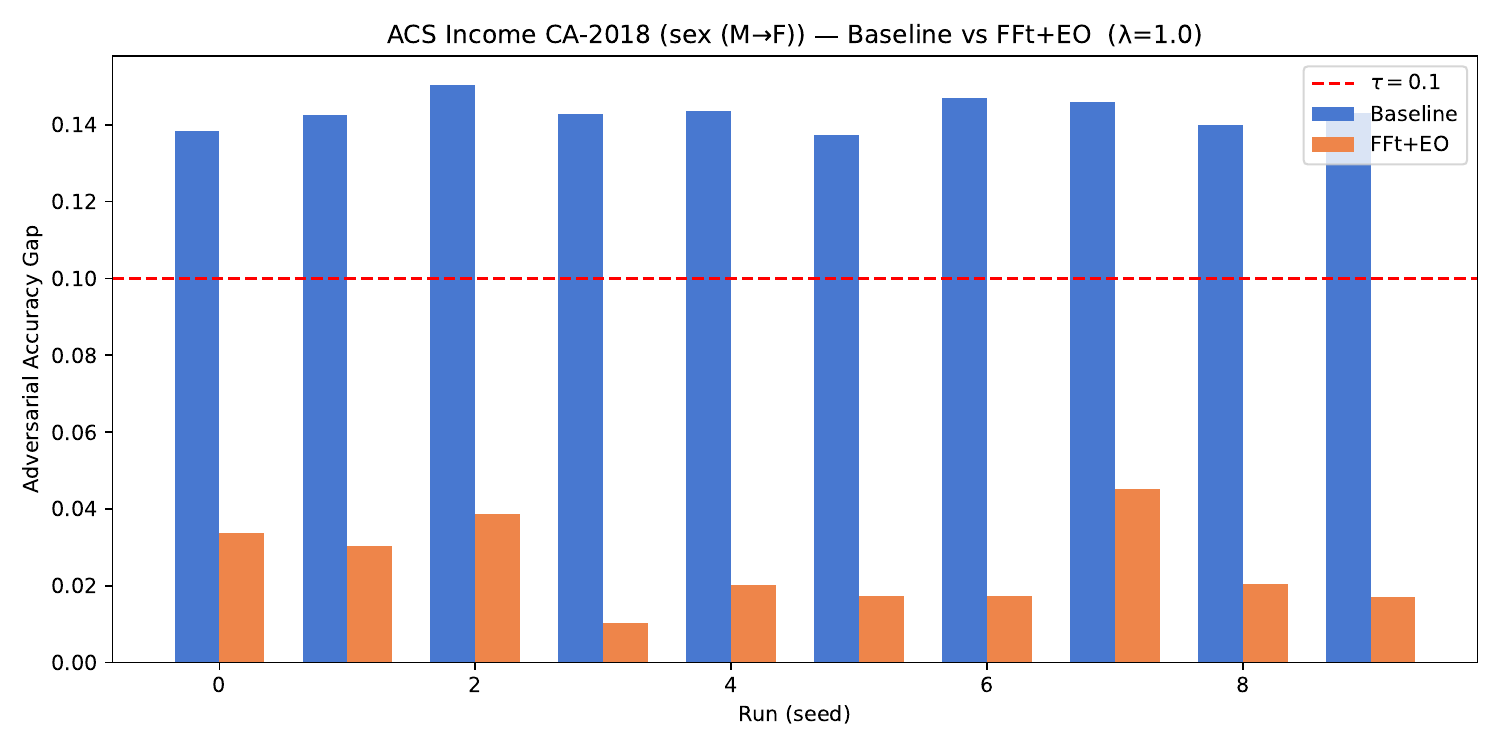}
        \caption{ACS (\texttt{sex}: male $\to$ female)}
        \label{fig:acs-sex}
    \end{subfigure}
    \hfill
    \begin{subfigure}{0.47\linewidth}
        \centering
        \includegraphics[width=\linewidth]{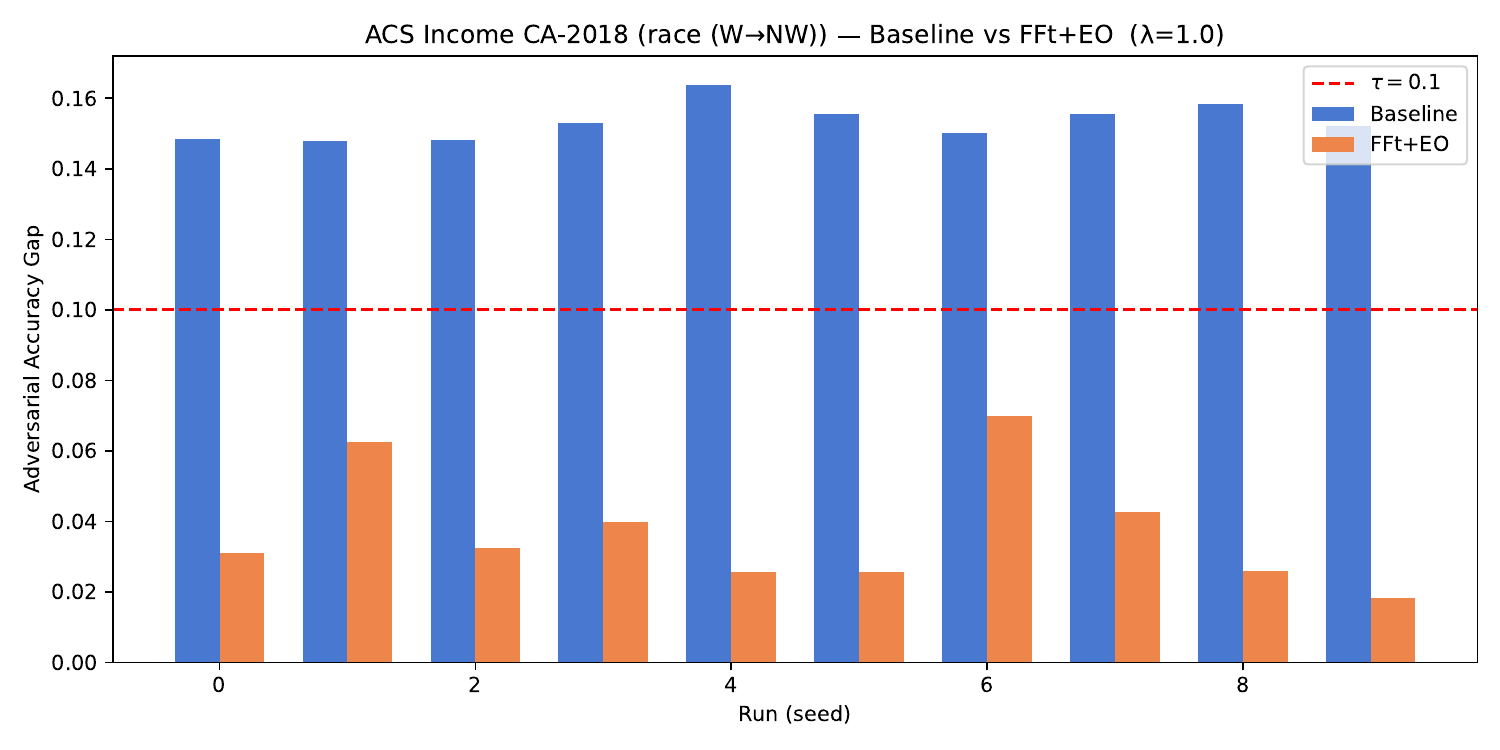}
        \caption{ACS (\texttt{race}: white $\to$ non-white)}
        \label{fig:acs-race}
    \end{subfigure}
    \caption{Adversarial accuracy gap for ACS Income CA-2018 across 10 runs. Rehearsal-based FFt brings the gap below $\tau=0.1$ in all 10 runs for both attributes.}
    \label{fig:acs}
\end{figure*}

\begin{figure*}[!th]
    \centering
    \begin{subfigure}{0.47\linewidth}
        \centering
        \includegraphics[width=\linewidth]{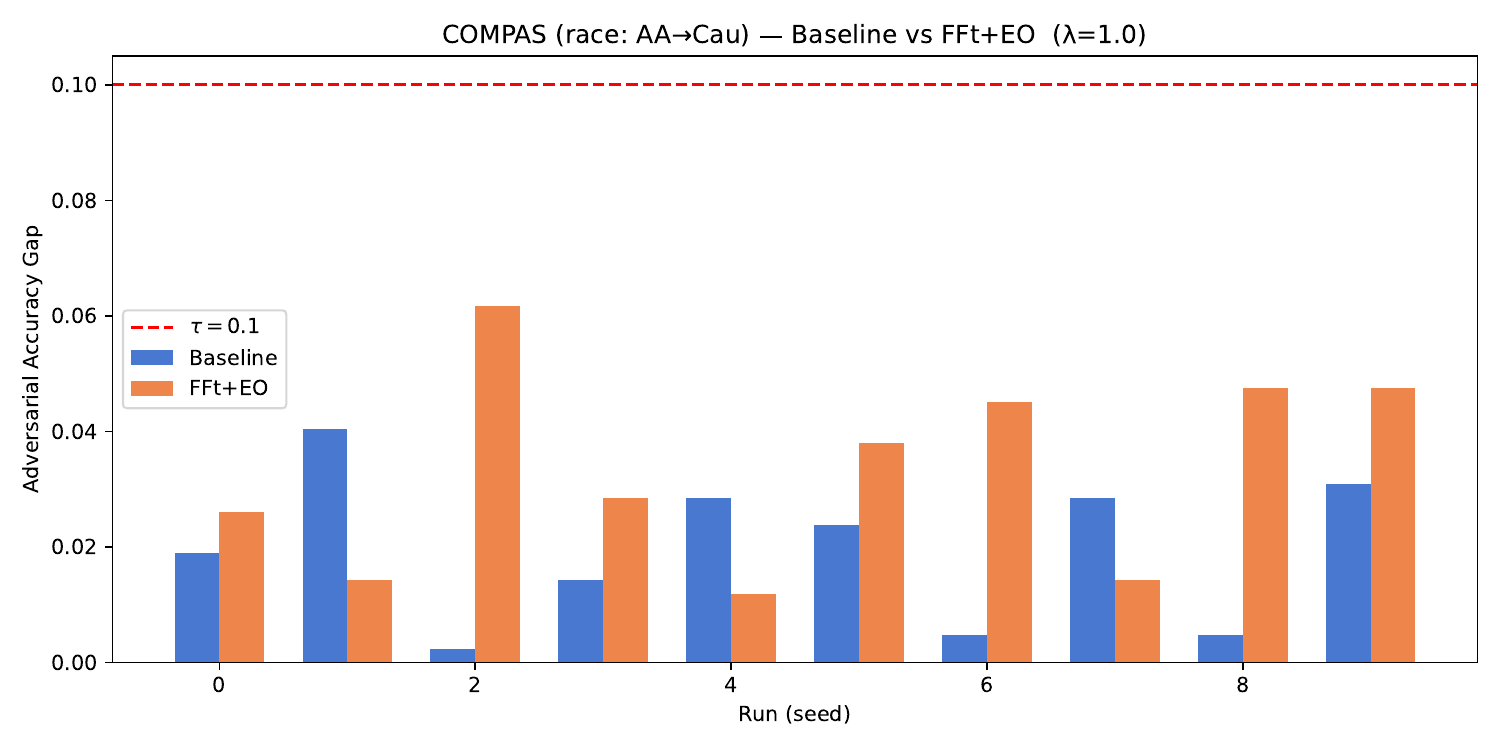}
        \caption{COMPAS (\texttt{race}: African-American $\to$ Caucasian)}
        \label{fig:compas-race}
    \end{subfigure}
    \hfill
    \begin{subfigure}{0.47\linewidth}
        \centering
        \includegraphics[width=\linewidth]{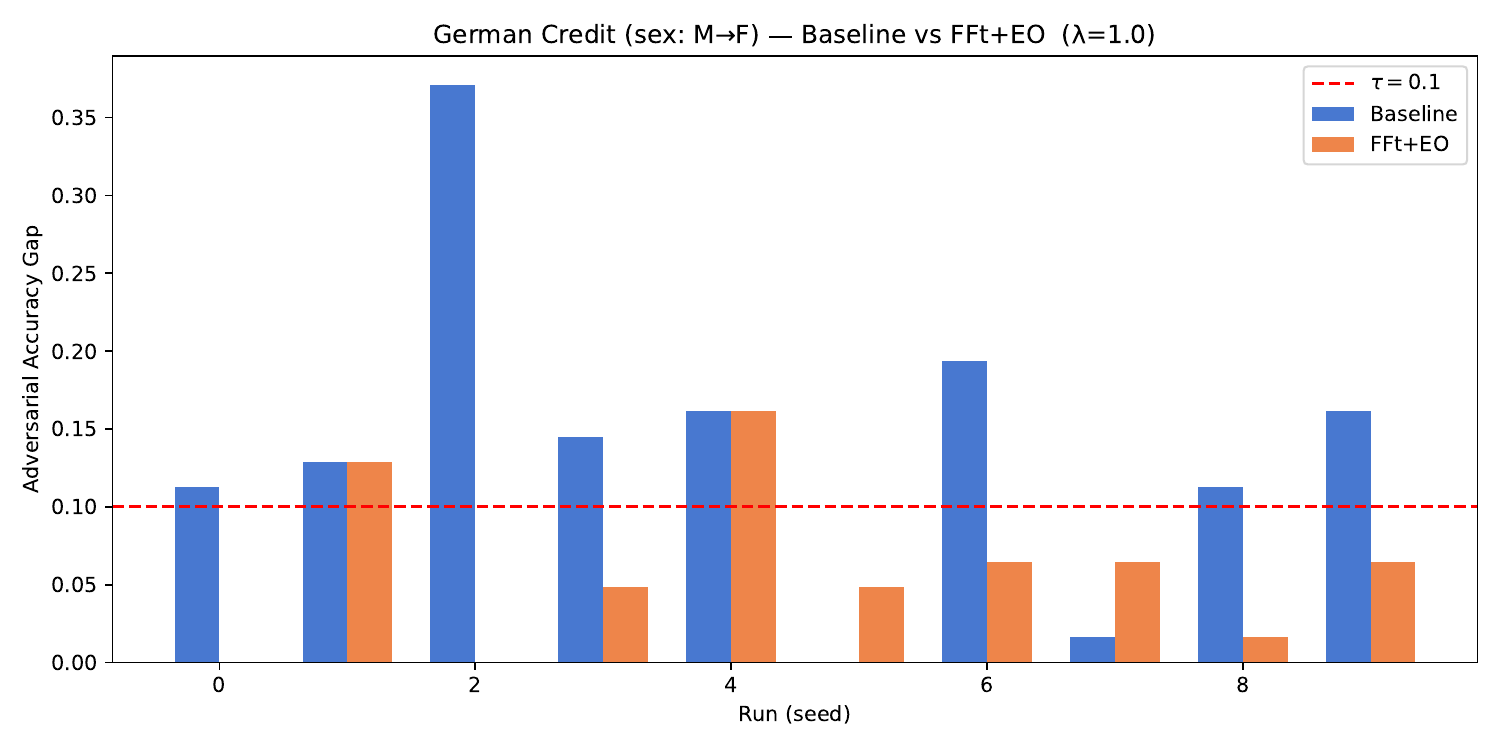}
        \caption{German Credit (\texttt{sex}: male $\to$ female)}
        \label{fig:german-sex}
    \end{subfigure}
    \caption{Adversarial accuracy gap for COMPAS and German Credit across 10 runs. COMPAS gaps are below $\tau$ in all runs; German Credit drops below $\tau$ in 8/10 runs.}
    \label{fig:compas}
\end{figure*}

\subsubsection{COMPAS (race)}

The Baseline gap is small ($\sim\!2.0\%$), which is already below the threshold $\tau$, because African-American and Caucasian defendants in this dataset have similar recidivism rates after standard filtering.  This is a case where the adversary cannot reliably distinguish the two groups even before any defense is applied.  Rehearsal-based FFt preserves this protection at $\sim\!3.4\%$ (below $\tau$ in all 10 runs); the small absolute increase is within run-to-run variance and does not reflect a degradation in security.  Crucially, the EO penalty substantially reduces $\Delta_{\text{EO}}$ from $\sim\!37.5\%$ to $\sim\!15.4\%$, tightening the theoretical bound of Theorem~\ref{thm:eqodds} by a factor of $2.4\times$ and providing a stronger formal guarantee even when the empirical gap was already negligible.

\subsubsection{German Credit (sex)}

The Baseline gap is $\sim\!14.0\%$: a model trained on male applicants classifies female applicants less accurately, reflecting genuine distributional differences in the German Credit data.  Rehearsal-based FFt reduces the gap to $\sim\!6.0\%$, below $\tau$ in 8 of 10 runs.  The 2.2:1 size asymmetry (690 male vs.\ 310 female) and the small dataset ($n\!=\!1000$) limit full convergence, consistent with the practical guideline in Section~\ref{sec:discussion}.

\subsubsection{UTKFaces (race)}

The Baseline gap is $\sim\!7.1\%$: a ResNet-18 trained exclusively on White faces generalises poorly to non-White faces whose age-distribution statistics differ.  FFt reduces the gap to $\sim\!5.5\%$, with both values below $\tau$ in all 10 runs.  The primary contribution of this setting is breadth: it demonstrates that the Adv--EO bound (Theorem~\ref{thm:eqodds}) holds for a deep convolutional model fine-tuned end-to-end, not just tabular MLPs, and that rehearsal ($\rho\!=\!0.2$) prevents the ResNet's larger parameter space from catastrophically forgetting the White-face distribution.  The $\Delta_{\text{EO}}$ reduction from $17.7\%$ to $15.5\%$ is more modest than in tabular settings, consistent with the tighter baseline gap leaving less room for EO improvement.


\subsubsection{Bias in Bios (sex)}

The Baseline gap is $\sim\!5.2\%$: an MLP trained on male biographies (TF-IDF features, gender-scrubbed text) predicts occupation less accurately on female biographies, reflecting the well-documented gender skew in Bias in Bios~\citep{deArteaga2019BiasIB}. Rehearsal-based FFt reduces the gap to $\sim\!0.9\%$, below $\tau$ in all 10 runs. This extends FFt to NLP classification and shows that the defence is not limited to tabular tasks.

\begin{figure*}[!t]
    \centering
    \begin{subfigure}[b]{0.30\linewidth}
        \centering
        \includegraphics[width=\linewidth]{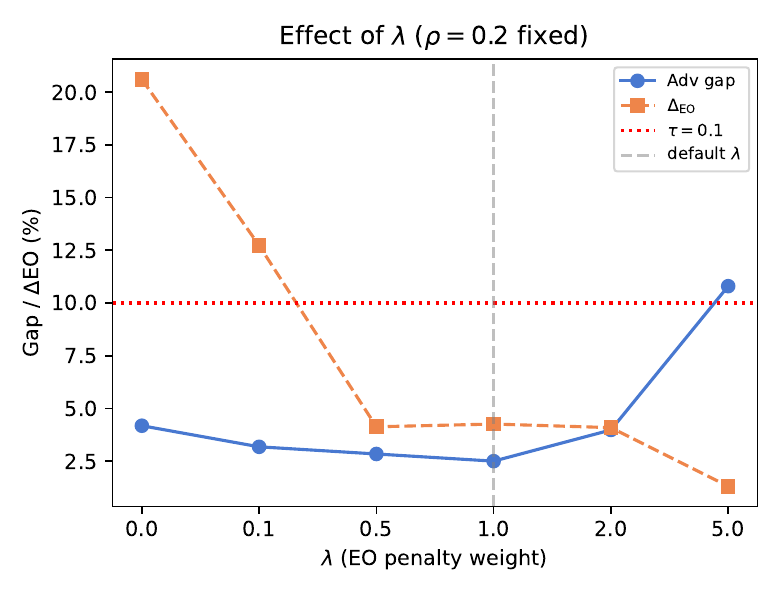}
        \caption{Effect of $\lambda$ ($\rho=0.2$ fixed). $\lambda=5.0$ overcorrects.}
        \label{fig:ablation-lambda}
    \end{subfigure}
    \hfill
    \begin{subfigure}[b]{0.30\linewidth}
        \centering
        \includegraphics[width=\linewidth]{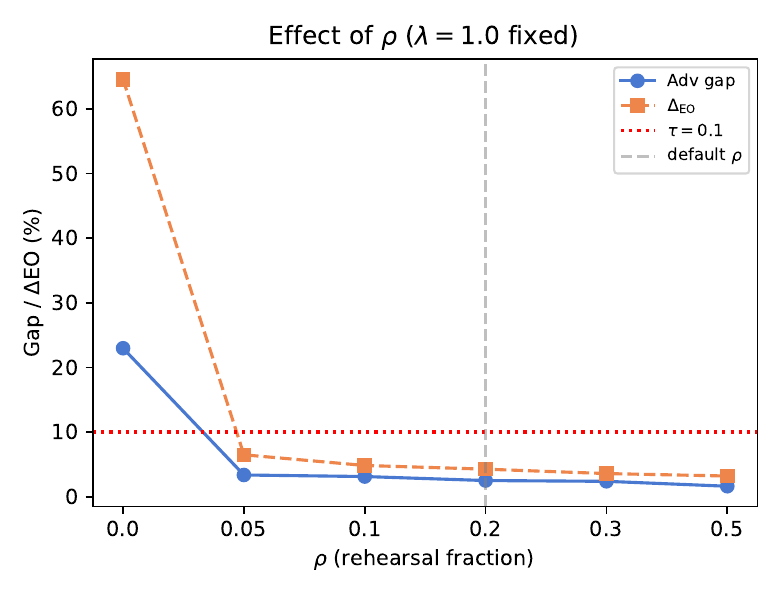}
        \caption{Effect of $\rho$ ($\lambda=1.0$ fixed). Any $\rho\ge0.05$ prevents forgetting.}
        \label{fig:ablation-rho}
    \end{subfigure}
    \hfill
    \begin{subfigure}[b]{0.30\linewidth}
        \centering
        \includegraphics[width=\linewidth]{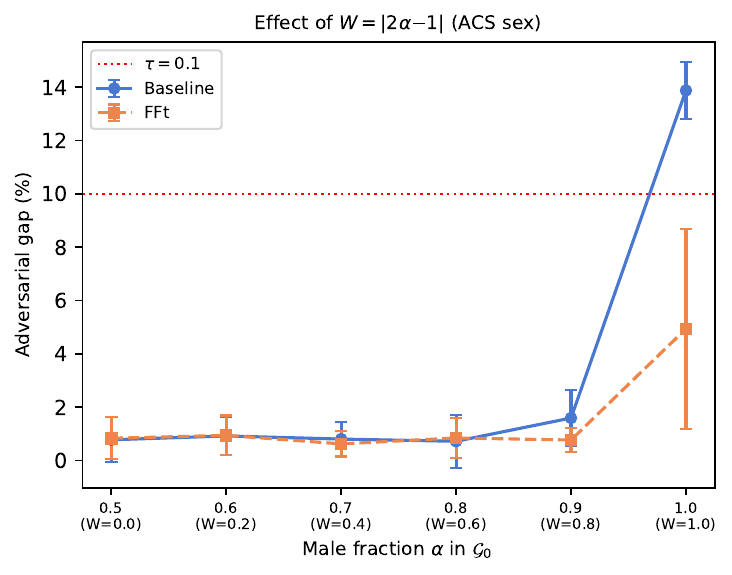}
        \caption{Effect of $W$ (mixed distribution, $\lambda=1.0$, $\rho=0.2$). Sharp jump at $\alpha=1.0$.}
        \label{fig:ablation-mixed}
    \end{subfigure}
    \caption{Ablation on ACS sex (10 seeds each). Blue = adversarial gap; orange = $\Delta_{\text{EO}}$ (4a, 4b); dotted red = $\tau=0.1$; dashed gray = default.}
    \label{fig:ablation}
\end{figure*}

\subsection{Discussion.}
\label{sec:discussion}

Our results are consistent with Theorem~\ref{thm:condition} and show a clear pattern across all evaluated settings.

\paragraph{Why rehearsal helps.}  Plain warm-start FFt can cause catastrophic forgetting — the fine-tuning step adapts the model so strongly to $\mathcal{G}_1$ that $\mathcal{G}_0$ accuracy collapses, increasing $\Delta_{\text{EO}}^{\text{fft}}$ rather than reducing it. Adding a rehearsal component (20\% of $\mathcal{G}_0$ data mixed into fine-tuning) prevents this collapse while still closing the gap: ACS achieves $2.5\%$ (sex) and $3.7\%$ (race), COMPAS $3.4\%$ (race), and German Credit $6.0\%$ (sex) — all below $\tau$ in the majority of runs.

\paragraph{Practical guideline.}  Use rehearsal-based FFt (mixing $\sim\!20\%$ of $\mathcal{G}_0$ into fine-tuning) to prevent catastrophic forgetting while closing the gap. Verify that $|\mathcal{S}_1|/|\mathcal{S}_0|$ is not too small; if the minority group is very small (e.g.\ $<20\%$ of $\mathcal{G}_0$), consider supplementing with synthetic data for $\mathcal{G}_1$. We also perform ablations on the sensitivity to $\lambda$, $\rho$, and the distributional shift $W$, as shown in Figure~\ref{fig:ablation}.

\paragraph{Adversary strength.}
Our experiments evaluate the Loss Test adversary~\citep{Suri2021FormalizingAE}, who observes only the model's test accuracy on each group.  This is the natural adversary for the black-box setting.  More powerful adversaries such as meta-classifiers trained on model weight snapshots~\citep{Suri2021FormalizingAE} or output-distribution statistics~\citep{dissecting23} can extract signals beyond the PPR gap. We do not evaluate against meta-classifiers and acknowledge this as a limitation.  Nevertheless, there is a structural reason to expect that EO enforcement degrades meta-classifier performance as well: meta-classifiers succeed by exploiting systematic differences in a model's group-conditional prediction patterns---precisely the signal that EO minimises.  A model with $\Delta_{\textup{EO}} \approx 0$ produces nearly identical conditional prediction rates for $\mathcal{G}_0$ and $\mathcal{G}_1$, removing the demographic signature that weight-space or logit-space meta-classifiers rely on.  Formally extending Theorem~\ref{thm:eqodds} to meta-classifier adversaries is an important open question we leave to future work.  The Loss Test is the standard black-box evaluation in the DIA literature and is directly comparable to prior work.


\subsection{Ablation: Hyperparameters $\lambda$, $\rho$, and Distributional Shift $W$}
\label{sec:ablation}

We ablate the two key hyperparameters of rehearsal-based FFt and the distributional shift weight $W$ on the ACS sex task ($\mathcal{G}_0$=male, $\mathcal{G}_1$=female), running 10 independent seeds per configuration.

\paragraph{Effect of $\lambda$ (Figure~\ref{fig:ablation-lambda}, $\rho=0.2$ fixed).}
$\lambda=0$ recovers plain rehearsal without an EO penalty; the gap is reduced (4.2\%) but $\Delta_{\text{EO}}$ stays large (20.6\%).  Increasing $\lambda$ monotonically reduces $\Delta_{\text{EO}}$: at our default $\lambda=1.0$ (marked~$\star$) the gap reaches 2.5\% with 10/10 runs below $\tau$.  However, $\lambda=5.0$ overcorrects, that is, the EO term dominates cross-entropy, widening the accuracy gap to 10.8\% and causing the bound to fail in 8 of 10 runs. Hence, we recommend choosing $\lambda\in[0.5,2.0]$.

\paragraph{Effect of $\rho$ (Figure~\ref{fig:ablation-rho}, $\lambda=1.0$ fixed).}
Without rehearsal ($\rho=0$) catastrophic forgetting causes the gap to jump to 23\% and $\Delta_{\text{EO}}$ to 64.5\%.  Any non-zero $\rho$ immediately prevents this: $\rho=0.05$ restores the gap to 3.4\% with 10/10 runs below threshold $\tau$.  Beyond $\rho=0.20$ gains are marginal (2.5\%$\to$1.6\% at $\rho=0.50$), therefore we use $\rho=0.20$ as the default throughout.

\paragraph{Effect of distributional shift $W$ (Figure~\ref{fig:ablation-mixed}).}
We vary the male fraction $\alpha\in\{0.5,\ldots,1.0\}$ in $\mathcal{G}_0$ (with $\mathcal{G}_1=(1{-}\alpha)\text{ male}+\alpha\text{ female}$), so $W\!=\!|2\alpha{-}1|$ ranges from $0$ to $1$.  For all $\alpha\!\in\![0.5,0.9]$ ($W\!\le\!0.8$), both baseline and FFt gaps are $\le\!1.6\%$ and already below $\tau$ in all 10 runs — the adversary has little signal to exploit.  At $\alpha\!=\!1.0$ ($W\!=\!1$, the biased protocol), the gap jumps sharply to $13.9\%$ for the baseline and $4.9\%$ for FFt (8/10 below $\tau$).  This nonlinear jump reflects the joint product $\Delta_{\text{EO}}\!\cdot\!W$: at $\alpha\!=\!1.0$ the model trained on pure males has never seen females, driving $\Delta_{\text{EO}}^{\text{base}}$ to $\sim\!78\%$, while at $\alpha\!=\!0.9$ the training set already includes $10\%$ females, keeping $\Delta_{\text{EO}}^{\text{base}}$ small.  This figure empirically confirms Corollary~\ref{cor:biased}: the biased protocol is qualitatively the hardest case, and any FFt defense that succeeds at $W\!=\!1$ is guaranteed to succeed at all $W\!<\!1$. Full numerical results for the $\lambda$ and $\rho$ sweeps are displayed in Appendix.

\section{Conclusion \& Future Work}

We introduce Fair Finetuning (FFt) with Equalized Odds (EO) as a principled defense against distribution inference attacks and develop a complete theoretical characterisation. Theorem~\ref{thm:eqodds} establishes the tighter bound $\text{Adv}(\mathcal{A},M_f) \le \Delta_{\text{EO}} \cdot W$, where $W$ is the distributional shift weight (Remark~\ref{rem:W}). Under the biased distribution protocol, where $W=1$, the bound reduces to $\text{Adv} \le \Delta_{\text{EO}}$, making FFt's effectiveness directly controllable through the EO constraint. Theorem~\ref{thm:condition} characterises exactly when FFt is beneficial and predicts failure when catastrophic forgetting raises $\Delta_{\text{EO}}^{\text{fft}}$ above $\Delta_{\text{EO}}^{\text{base}}$ while Proposition~\ref{prop:tight} confirms the bound is tight. Corollary~\ref{cor:biased} establishes the biased protocol as the worst case for the defender and Proposition~\ref{prop:groupsize} identifies the group-size failure regime.

Empirically, we evaluate a biased distribution setup (one pure demographic group as $\mathcal{G}_0$, the complementary group as $\mathcal{G}_1$) across six datasets spanning three modalities---tabular (ACS Income, COMPAS, German Credit), image (UTKFaces), and NLP (Bias in Bios, LSAC), with LSAC evaluated separately in the Appendix. Rehearsal-based FFt reduces the adversarial gap to $2.5\%$ (ACS sex) and $3.7\%$ (ACS race)---below $\tau$ in all 10 runs. COMPAS's baseline gap of $2.0\%$ remains below $\tau$ at $3.4\%$ after FFt, with the EO penalty tightening $\Delta_{\text{EO}}$ from $37.5\%$ to $15.4\%$. German Credit drops from $14.0\%$ to $6.0\%$ (8/10 runs below $\tau$); UTKFaces from $7.1\%$ to $5.5\%$ (all runs); and Bias in Bios from $5.2\%$ to $0.9\%$ (all runs). On LSAC, FFt reduces the gap from $\sim\!74\%$ to $\sim\!8.8\%$ (sex, below $\tau$ in the majority of runs); for race (5:1 size asymmetry), FFt halves the gap ($49.8\% \to 25.2\%$) but cannot close it entirely---an instance of the failure condition in Theorem~\ref{thm:condition}. These results confirm that rehearsal-based FFt generalises across datasets, modalities, and sensitive attributes. A mixed-distribution ablation (Figure~\ref{fig:ablation-mixed}) additionally validates Corollary~\ref{cor:biased}: for $\alpha\in[0.5,0.9]$ (i.e.\ $W\le 0.8$) the adversarial gap is $\le\!1.6\%$ and already below $\tau$, with a sharp jump to $13.9\%$ only at $\alpha=1.0$ ($W=1$), confirming that the biased protocol is the qualitatively hardest case.

\subsection{Limitations}

\begin{enumerate}
    \item \textbf{\underline{Data Acquisition.}} FFt requires sampling from the complementary distribution $\mathcal{G}_{\neg b}$. While this distribution is publicly known under the standard threat model, obtaining a labeled sample from it still incurs practical cost. Future work should explore whether synthetic data generation or differentially private mechanisms can substitute for real complementary samples.

    \item \textbf{\underline{Attribute Awareness.}} The current formulation assumes the defender knows which sensitive attribute the adversary targets. In practice, an adversary could probe any undisclosed attribute (e.g.\ income, political affiliation). Extending FFt to simultaneously enforce EO across multiple sensitive attributes---or to use attribute-agnostic representation learning---is an important direction (a theoretical extension to multiple attributes is given in the Appendix).

    \item \textbf{\underline{Group Size Asymmetry.}} When $|\mathcal{G}_1| \ll |\mathcal{G}_0|$, fine-tuning on $\mathcal{G}_1$ cannot fully recalibrate the model (LSAC race). Complementary approaches such as synthetic oversampling of $\mathcal{G}_1$ or importance-weighted EO constraints should be explored.

    \item \textbf{\underline{Dataset Scope.}} While our evaluation covers tabular, image, and NLP modalities across six datasets, it is limited to binary classification tasks with binary group membership. Extending FFt to multi-class sensitive attributes, continuous protected attributes, and larger-scale vision-language models remains important future work.

    \item \textbf{\underline{Adversary Strength.}} Experiments evaluate the black-box Loss Test adversary.  Meta-classifier adversaries that operate on model weight distributions~\citep{Suri2021FormalizingAE} may extract signals beyond the PPR gap; we provide a structural argument for why EO enforcement should degrade their performance (Section~\ref{sec:discussion}), but a formal extension of Theorem~\ref{thm:eqodds} to meta-classifiers and empirical validation remain important open questions.

    \item \textbf{\underline{Comparison to Existing Defenses.}} We include a direct comparison to DP-SGD (results in Appendix).  Direct comparison to property unlearning~\citep{secrypt23} and Inf2Guard~\citep{leila24usenix} is left as future work; we anticipate FFt will be complementary, since it is orthogonal (fairness-based post-processing, no cryptographic overhead) and the two defenses target different leakage surfaces.
\end{enumerate}

\subsection{Future Directions}

Future work may explore how fairness constraints can be embedded directly in the training objective rather than applied as post-hoc fine-tuning, whether the Adv--EO bound of Theorem~\ref{thm:eqodds} can be tightened for specific model families, and how rehearsal-based FFt interacts with differential privacy guarantees.

\section*{Impact Statement}

\paragraph{Ethical Considerations.}
All datasets used in this work (ACS Income, COMPAS, German Credit, UTKFaces, Bias in Bios, and LSAC) are publicly available benchmarks; no new data was collected and no human subjects were recruited, so IRB review was not required.  That said, several datasets carry ethical weight that warrants acknowledgement.  COMPAS is a commercial recidivism risk-score instrument with well-documented racial disparities~\citep{angwin16compas} and is used in real criminal-sentencing contexts; treating it purely as a benchmark normalizes its deployment and may implicitly lend credibility to its continued use.  UTKFaces encodes race and gender through third-party annotation, embedding socially constructed categories as ground-truth labels.  Across all datasets we operationalize race and sex as binary attributes, which reflects the limitations of existing benchmarks rather than a normative claim about identity, but risks invisibilizing non-binary and multiracial individuals in the evaluation.  We flag this as a known gap that future work should address by sourcing datasets with richer demographic representation.

\paragraph{Researcher Positionality.}
This work sits at the intersection of ML security and algorithmic fairness, and the author's training in both areas shapes several design choices.  Framing the problem as a cryptographic game (adversarial advantage, bit-guessing) reflects a security-first lens that defines harm as information leakage and success as bounding an adversary's distinguishing advantage.  This framing is powerful for formal guarantees but may underweight harms that are harder to quantify, such as the degrading effects of surveillance or the symbolic harm of demographic profiling.  The choice of Equalized Odds as the operative fairness criterion reflects a group-statistical view of fairness; individual or causal notions of fairness would yield different bounds and potentially different empirical conclusions.  The author approaches the problem with the assumption that fairness and privacy are complementary objectives — a perspective that motivates the central thesis but may occlude settings where reducing distributional leakage conflicts with other fairness criteria (e.g.\ individual fairness or counterfactual fairness).

\paragraph{Adverse Impact.}
The primary intended use of FFt is defensive: reducing an adversary's ability to infer demographic composition from a deployed model.  However, several misuse and misapplication risks deserve explicit attention.  First, the defense could engender false confidence if applied without checking the preconditions established in Theorems~\ref{thm:eqodds} and~\ref{thm:condition}: a practitioner who applies FFt with severe group-size asymmetry ($\alpha \ll 0.2$) or an over-large EO penalty ($\lambda \gg 2$) may produce a model whose adversarial gap remains high or whose accuracy-EO coupling breaks, yet who reports the deployment as ``protected.''  Second, FFt could be exploited as a form of fairness-washing: an organization could apply a minimal EO constraint ($\lambda$ just large enough to reduce the measured adversarial gap below $\tau$) while the underlying model continues to encode discriminatory representations along unmeasured attributes.  The bound $\text{Adv} \le \Delta_{\text{EO}}$ applies specifically to the positive-prediction-rate gap under black-box access; it does not protect against white-box attacks, side-channel leakage via confidence scores or latency, or inference about sensitive attributes not included in the EO constraint.  Deployers should treat the bound as a guarantee about one specific attack surface, not a blanket privacy certificate.

\bibliography{files/refs}

\appendix
\section{Appendix}
\subsection{Extension to Multiple Sensitive Attributes}
\label{app:multiattr}

The single-attribute bound of Theorem 1 extends directly to $K$ simultaneous sensitive attributes by applying the proof independently per attribute.

\begin{theorem}[Multi-attribute Extension]
\label{thm:multi}
Suppose $M_f$ satisfies EO with disparity $\Delta_{\textup{EO}}^{(k)}$ for each of $K$ sensitive attributes $A_1,\ldots,A_K$, enforced simultaneously.  Then for any adversary targeting attribute $A_k$:
\[
\textup{Adv}(\mathcal{A}_{A_k}, M_f) \;\le\; \Delta_{\textup{EO}}^{(k)} \cdot W^{(k)},
\]
where $W^{(k)}$ is defined as in Theorem 1 with respect to $A_k$.  The bounds across attributes are independent: suppressing leakage about $A_k$ does not require sacrificing protection for $A_{k'}$.
\end{theorem}

\begin{proof}[Proof sketch]
Theorem 1 is proved for a single attribute by conditioning on $A$.  The proof is attribute-specific: the cancellation $\Delta_y = \delta_y \Delta P_y$ depends only on $P_{y,a}^{(k)}$ (the model's prediction rates conditioned on $A_k$).  Applying the same argument independently for each $k$ yields the stated bounds.  Simultaneous EO enforcement on $K$ attributes sets $|\delta_y^{(k)}| \le \Delta_{\textup{EO}}^{(k)}$ for every $k$, completing the proof.
\end{proof}

\section{Proof of the Identity $\Delta P_{y,0} = -\Delta P_{y,1}$}
\label{app:identity-proof}

For any fixed value of the true label $Y=y$ and distribution $D$, the sensitive attribute
$A \in \{0,1\}$ satisfies
\[
\Pr[A=0 \mid Y=y, D] + \Pr[A=1 \mid Y=y, D] = 1.
\]
Define the distributional differences
\[
\Delta P_{y,a} 
:= 
\Pr[A=a \mid Y=y, D_1] 
- 
\Pr[A=a \mid Y=y, D_0].
\]

Applying the above identity to both $D_1$ and $D_0$ gives
\[
\Pr[A=0 \mid Y=y, D_1] + \Pr[A=1 \mid Y=y, D_1] = 1,
\]
\[
\Pr[A=0 \mid Y=y, D_0] + \Pr[A=1 \mid Y=y, D_0] = 1.
\]
Subtracting the first equation from the second yields
\begin{multline}
\big(\Pr[A=0 \mid Y=y, D_1] - \Pr[A=0 \mid Y=y, D_0]\big) + \\
\big(\Pr[A=1 \mid Y=y, D_1] - \Pr[A=1 \mid Y=y, D_0]\big) = 0, \nonumber
\end{multline}

which is equivalent to
\[
\Delta P_{y,0} + \Delta P_{y,1} = 0.
\]
Thus,
\[
\Delta P_{y,0} = - \Delta P_{y,1}.
\]

In other words, because $A$ is binary and its conditional probabilities must sum to one under
each distribution, i.e., any increase in the probability of $A=1$ among individuals
with label $Y=y$ under $D_1$ must produce an equal decrease in the probability of
$A=0$.
Hence the two changes are exact opposites.
It is therefore convenient to define

\[
\Delta P_y := \Delta P_{y,1},
\qquad\text{so that}\qquad
\Delta P_{y,0} = -\Delta P_y.
\]

\section{Scope of the PPR Bound and the Base-Rate Floor}
\label{app:ppr-scope}

Theorem 1 bounds the \emph{positive-prediction-rate} (PPR) gap
$|\Pr[\hat{Y}=1 \mid D_1] - \Pr[\hat{Y}=1 \mid D_0]|$.
The Loss Test adversary, however, observes the \emph{accuracy} gap
$|acc(M, S_0^{\text{test}}) - acc(M, S_1^{\text{test}})|$.
This section derives the exact relationship between the two and explains the base-rate floor that appears when $\Delta_{\text{EO}}=0$.

\paragraph{Accuracy in terms of TPR and FPR.}
For a binary classifier on distribution $D$ with positive base rate $\pi = \Pr[Y=1 \mid D]$:
\[
acc(M, D) = \pi \cdot \text{TPR}(D) + (1-\pi) \cdot (1 - \text{FPR}(D)),
\]
where $\text{TPR}(D) = \Pr[\hat{Y}=1 \mid Y=1, D]$ and $\text{FPR}(D) = \Pr[\hat{Y}=1 \mid Y=0, D]$.

\paragraph{Accuracy gap when EO is perfectly enforced.}
Suppose $\Delta_{\text{EO}} = 0$, so that TPR and FPR are identical across groups:
$\text{TPR}(D_0) = \text{TPR}(D_1) = \text{TPR}$ and $\text{FPR}(D_0) = \text{FPR}(D_1) = \text{FPR}$. Let $\Delta_{\text{acc}} = \text{acc}(M,D_1) - \text{acc}(M,D_0)$. Then:
\begin{align}
    \Delta_{\text{acc}}
    &= \pi_1\cdot\text{TPR} + (1-\pi_1)(1-\text{FPR}) \notag\\
    &\phantom{=} - \pi_0\cdot\text{TPR} - (1-\pi_0)(1-\text{FPR}) \notag\\
    &= (\pi_1-\pi_0)\cdot\bigl(\text{TPR} + \text{FPR} - 1\bigr) \notag\\
    &= \Delta\pi\cdot\bigl(\text{TPR} + \text{FPR} - 1\bigr).
\end{align}
where $\Delta\pi = \pi_1 - \pi_0 = \Pr[Y=1 \mid D_1] - \Pr[Y=1 \mid D_0]$ is the label base-rate difference between groups.

This quantity is the \emph{base-rate floor}: the accuracy gap that persists even after perfect EO enforcement, arising purely from the two groups having different label prevalences.  It is zero only when $\Delta\pi=0$ (equal base rates) or when $\text{TPR}+\text{FPR}=1$ (the model is no better than random).

\paragraph{Interpretation.}
Theorem 1 bounds the PPR-driven component of adversarial distinguishability; the base-rate floor is a separate, EO-independent contribution.  When groups have similar label distributions ($|\Delta\pi|$ small), the floor is negligible and the bound is informative for the accuracy gap as well.  When base rates diverge sharply, the floor dominates even at $\Delta_{\text{EO}}=0$, which is precisely the mechanism behind the $\lambda=5.0$ failure (Figure 4a) and the LSAC \texttt{race} partial failure (shown below).

\paragraph{Meta-classifier adversaries.}
Adversaries more powerful than the Loss Test (e.g.\ meta-classifiers trained on model weight snapshots~\citep{Suri2021FormalizingAE}) may extract signals beyond both the PPR gap and the accuracy gap.  Bounding their advantage requires tighter coupling with the model's full output distribution, which we leave as future work.

\section{Distribution Inference Attack Game}
\label{app:orig_threat}

The original distribution inference attack game as introduced by~\cite{Suri2021FormalizingAE} is displayed in Figure~\ref{fig:orig_threat}.

\section{LSAC Results (Near-Saturated Regime)}
\label{app:lsac}

The LSAC bar-passage dataset~\citep{Wightman1998LSACNB} contains records for 20,800 law-school students. Features include LSAT score, undergraduate GPA, family income, age, gender, race, law-school cluster (1--6), and full/part-time status. The binary target is bar passage. We experiment on two sensitive attributes: \texttt{race} (\texttt{white} vs.\ \texttt{non-white}) and \texttt{sex} (\texttt{male} vs.\ \texttt{female}).

LSAC represents the \emph{near-saturated} regime: the MLP trained on $\mathcal{G}_0$ achieves $\sim\!99.9\%$ accuracy, causing $\Delta_{\text{EO}} \approx 1$. Plain warm-start FFt (without rehearsal) is used here; rehearsal would preserve the near-perfect $\mathcal{G}_0$ accuracy and keep the gap large.

\begin{table}[!ht]
\centering
\caption{LSAC adversarial gap and EO disparity (mean over 10 runs, $W=1$). The bound $\text{Adv} \le \Delta_{\text{EO}}$ holds in both rows.}
\label{tab:lsac}
\resizebox{\columnwidth}{!}{%
\begin{tabular}{lcccccc}
\toprule
\multirow{2}{*}{\textbf{Setting (G0$\to$G1)}} &
  \multicolumn{2}{c}{\textbf{Adv gap}} &
  \multicolumn{2}{c}{\textbf{$\Delta_{\text{EO}}$}} &
  \multirow{2}{*}{\textbf{${\le}\,\tau$?}} &
  \multirow{2}{*}{\textbf{Bound holds?}} \\
\cmidrule(lr){2-3}\cmidrule(lr){4-5}
 & Base & FFt & Base & FFt & & \\
\midrule
LSAC \texttt{sex}\ (M\,$\to$\,F)   & $74.3\%$ & $\mathbf{8.8\%}$  & $99.9\%$ & $11.8\%$ & 6/10   & \cmark \\
LSAC \texttt{race}\ (W\,$\to$\,NW) & $49.8\%$ & $\mathbf{25.2\%}$ & $99.9\%$ & $68.7\%$ & \xmark & \cmark \\
\bottomrule
\end{tabular}}
\end{table}

\paragraph{LSAC (sex).}
The Baseline gap is very large ($\sim\!74\%$): a model trained exclusively on male students generalises poorly to female students whose bar-passage rates differ significantly. Plain warm-start FFt reduces the gap to $\sim\!8.8\%$, below $\tau$ in 6 of 10 runs.

\paragraph{LSAC (race).}
The Baseline gap is $\sim\!50\%$: a model trained on white students performs near-randomly on non-white students. FFt reduces the gap to $\sim\!25\%$ — a substantial improvement, but still above $\tau$. The 5.3:1 size asymmetry (17,493 white vs.\ 3,307 non-white) combined with a 33-percentage-point label-rate difference (whites $\sim\!85\%$ vs.\ non-whites $\sim\!52\%$ bar passage) prevents full recalibration.

\begin{figure*}[!th]
    \centering
    \begin{subfigure}{0.47\linewidth}
        \centering
        \includegraphics[width=\linewidth]{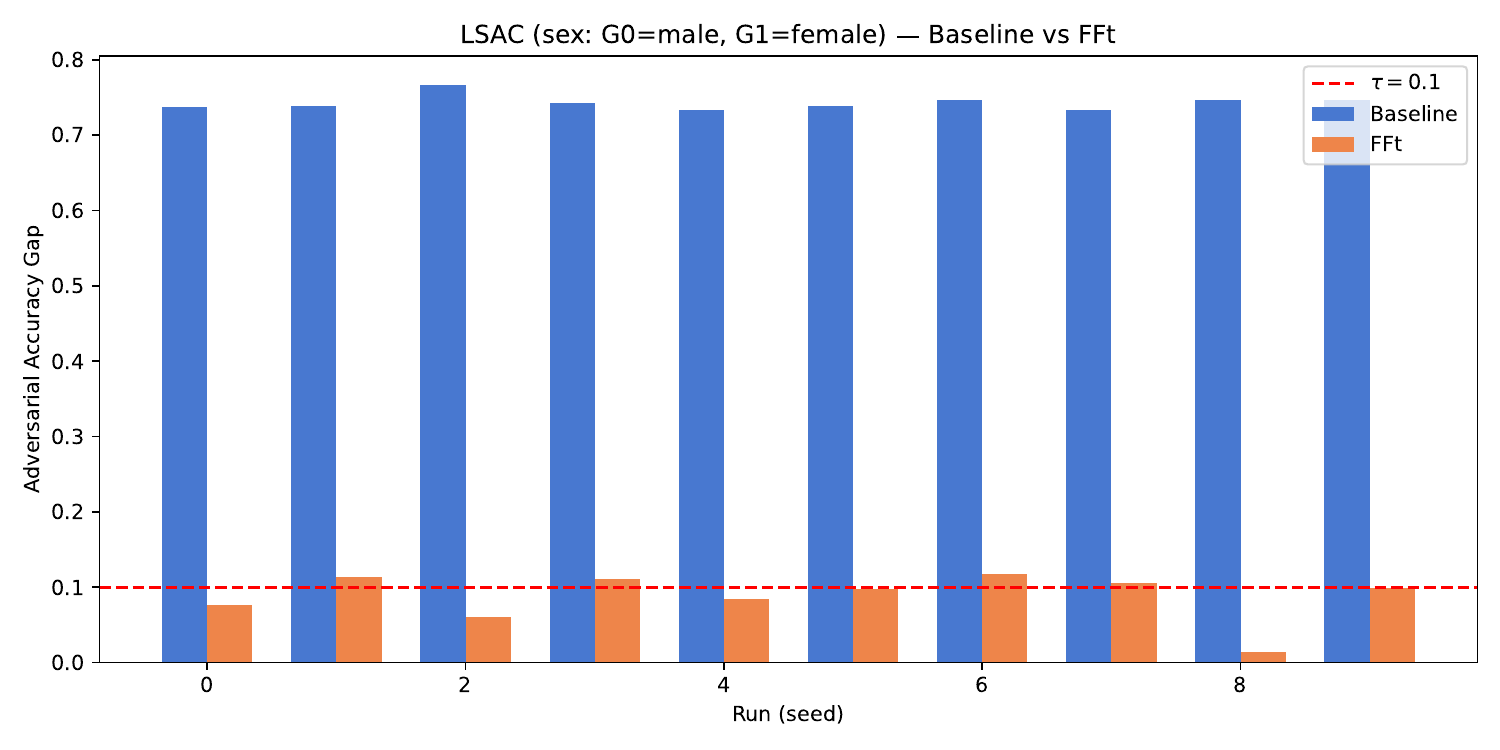}
        \caption{LSAC (\texttt{sex}: male $\to$ female)}
        \label{fig:lsac-sex}
    \end{subfigure}
    \hfill
    \begin{subfigure}{0.47\linewidth}
        \centering
        \includegraphics[width=\linewidth]{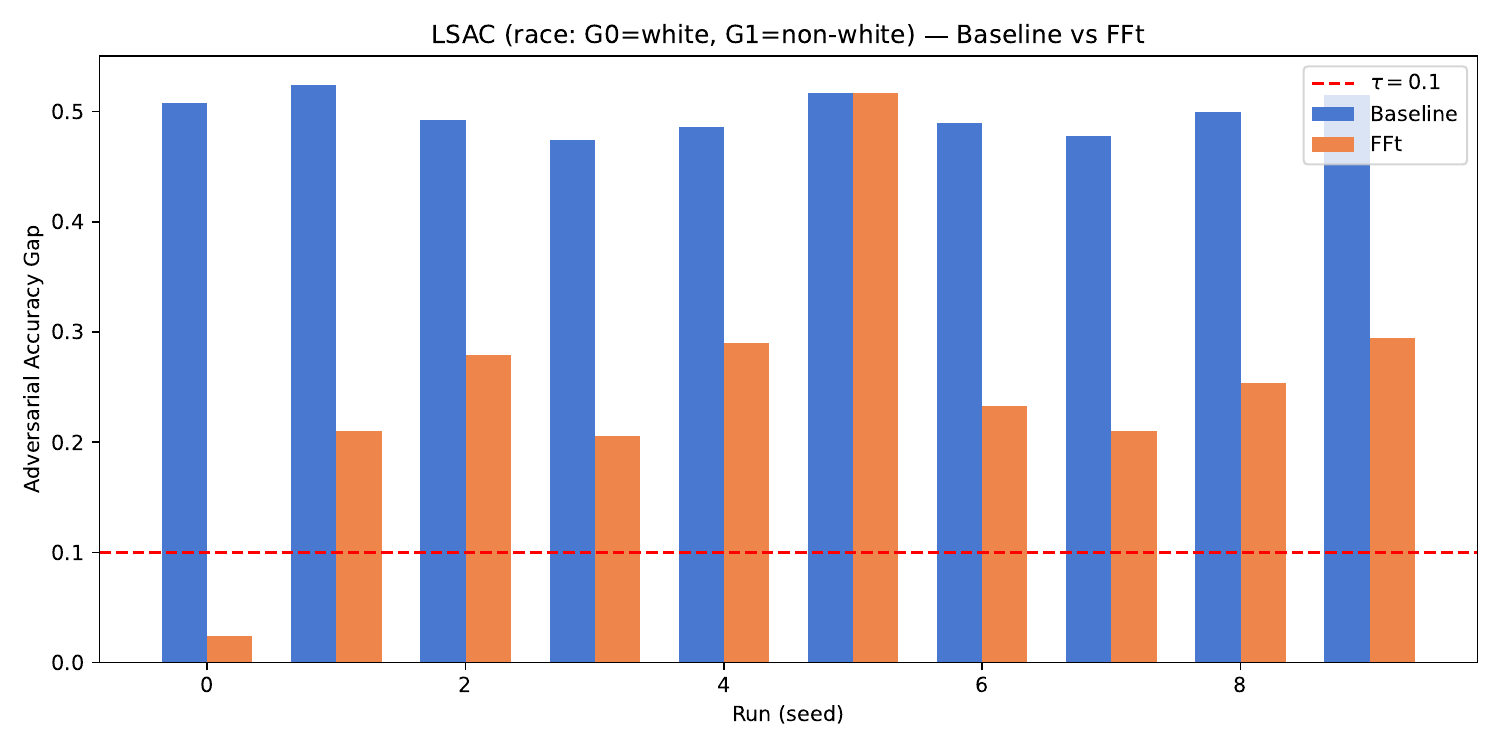}
        \caption{LSAC (\texttt{race}: white $\to$ non-white)}
        \label{fig:lsac-race}
    \end{subfigure}
    \caption{Adversarial accuracy gap for LSAC (mean over 10 runs, blue=Baseline, orange=FFt). FFt reduces the sex gap below $\tau=0.1$ in the majority of runs. For race, the 5:1 group-size asymmetry limits FFt's effectiveness.}
    \label{fig:lsac}
\end{figure*}

\section{Why $\Delta_{\text{EO}} \approx 1$ in the Near-Saturated Regime}
\label{app:deo-lsac}

The $\Delta_{\text{EO}} \approx 99.9\%$ entries in the LSAC rows of Table~\ref{tab:lsac} may appear surprising, but they follow directly from the definition.  A model that achieves near-perfect accuracy ($\sim\!99.9\%$) on $\mathcal{S}_0$ has calibrated its conditional prediction rates almost exactly to the $\mathcal{G}_0$ distribution: $P(\hat{Y}=1 \mid Y=1,\mathcal{G}_0) \approx 1$ (near-perfect true-positive rate) and $P(\hat{Y}=1 \mid Y=0,\mathcal{G}_0) \approx 0$ (near-perfect true-negative rate).  When this model is evaluated on $\mathcal{S}_1$ (a group it \emph{never saw during training}), it imposes its $\mathcal{G}_0$-fitted decision boundary on $\mathcal{G}_1$'s feature distribution.  Because $\mathcal{G}_1$ has systematically different feature patterns (different LSAT/GPA profiles for non-white students; different score ranges for female students), the $\mathcal{G}_0$-trained boundary classifies most $\mathcal{G}_1$ samples incorrectly, yielding a near-zero TPR on $\mathcal{G}_1$.  The EO disparity is therefore

\begin{align}
    \Delta_{\text{EO}} &= \max_y\bigl|
        P(\hat{Y}=1 \mid Y=y,\mathcal{G}_0) \notag\\
    &\phantom{{}=\max_y\bigl|}
        - P(\hat{Y}=1 \mid Y=y,\mathcal{G}_1)\bigr| \notag\\
    &\approx |1 - 0| \approx 1.
\end{align}

This is not a pathological result. In fact, it is the \emph{expected consequence of exclusive single-group training}: a model that is maximally accurate for $\mathcal{G}_0$ is simultaneously maximally unfair toward $\mathcal{G}_1$ by the EO metric.  The adversarial gaps ($74.3\%$ for sex; $49.8\%$ for race) are correctly bounded \emph{below} $\Delta_{\text{EO}} = 99.9\%$, consistent with Theorem 1; the bound is simply loose here because the adversary's observable signal (accuracy difference) is a coarser measure than population-level EO disparity.  After FFt, $\Delta_{\text{EO}}$ falls dramatically to $11.8\%$ (sex) and $68.7\%$ (race).  Theorem 1 then guarantees that the new adversarial advantage is bounded by these lower values and this is empirically confirmed by the post-FFt gaps ($8.8\%$ and $25.2\%$).  For LSAC sex the post-FFt bound is nearly tight ($8.8\% \le 11.8\%$), illustrating that the residual leakage is almost entirely explained by the residual unfairness $\delta_y$, which is the sole surviving term in the proof's algebraic cancellation.

\begin{figure*}[t]
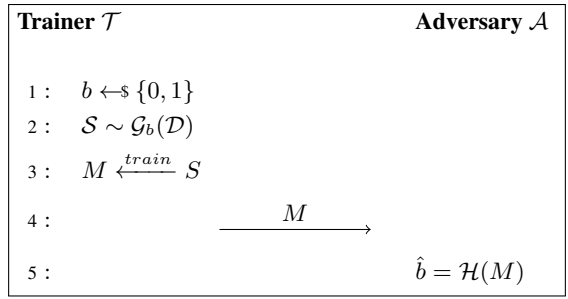

    \centering
    \fbox{%
    \pseudocode[]{%
        \textbf{Trainer } \mathcal{T} \<\< \textbf{Adversary } \adv \\[][]
         \\
        \pcln b \sample \bin \\
        \pcln \mathcal{S} \sim \mathcal{G}_b(\mathcal{D}) \\
        \pcln M \xleftarrow{train}{} S \\
        \pcln \< \sendmessageright{top={$M$}, length=2cm} \\
        \pcln \< \< \hat{b} = \mathcal{H}(M)
    }
}
\caption{The traditional distribution inference attack setting, as described in~\cite{Suri2021FormalizingAE}. Step 2 involves sampling a dataset $\mathcal{S}$ from the function $\mathcal{G}_b$ (where $b$ is the selected bit from Step 1). Then, Step 3 involves training the model on this sampled dataset $\mathcal{S}$. Finally, this trained model is released to the public and the adversary applies a hypothesis $\mathcal{H}(\cdot)$ to recover the training distribution via computing $\hat{b}$.}
\label{fig:orig_threat}
\end{figure*}

\section{Utility Impact}

Table~\ref{tab:utility} reports $\mathcal{G}_0$ test accuracy before and after FFt.
The utility cost tracks how aggressively the EO constraint must be enforced: ACS sex incurs a ${\sim}4\%$ drop ($79.6\%{\to}75.7\%$), while ACS race incurs a larger ${\sim}9\%$ drop ($79.7\%{\to}71.1\%$), reflecting the stronger weight updates needed to close a bigger EO gap ($78.6\%{\to}5.5\%$ vs.\ $78.0\%{\to}4.3\%$ for sex).
COMPAS loses ${\sim}4\%$ ($68.2\%{\to}63.8\%$).
German Credit is the outlier: accuracy barely changes ($70.3\%{\to}69.7\%$, $\le\!1\%$ delta).
On this 1{,}000-sample dataset the model converges in fewer iterations and rehearsal is especially effective at preserving $\mathcal{G}_0$ accuracy, since even a small rehearsal buffer ($\rho{=}0.2$) covers a proportionally large fraction of the training distribution.
Across all settings the utility cost is within the range routinely accepted for fairness interventions~\citep{hardt16}, and it is directly offset by the EO reductions shown in Table 1 that close the adversarial gap.

\begin{table}[!ht]
\centering
\caption{$\mathcal{G}_0$ test accuracy before and after FFt (mean, 10 runs), for the four tabular/structured settings where architectures are comparable.  UTKFaces (ResNet-18) and Bias in Bios (TF-IDF+MLP) use different model families; their per-epoch accuracy curves are not directly comparable to the MLP training curve and are omitted here.}
\label{tab:utility}
\resizebox{\columnwidth}{!}{%
\begin{tabular}{lcc}
\toprule
\textbf{Setting (G0\,$\to$\,G1)} & \textbf{Baseline} & \textbf{FFt} \\
\midrule
ACS \texttt{sex}\ (M\,$\to$\,F)             & $79.6\%$ & $75.7\%$ \\
ACS \texttt{race}\ (W\,$\to$\,NW)           & $79.7\%$ & $71.1\%$ \\
COMPAS \texttt{race}\ (AA\,$\to$\,Cau)      & $68.2\%$ & $63.8\%$ \\
German \texttt{sex}\ (M\,$\to$\,F)          & $70.3\%$ & $69.7\%$ \\
\bottomrule
\end{tabular}}
\end{table}

\section{Ablation Hyperparameter Tables}
\label{app:ablation-tables}

The hyperparameters for the ablation study is shown in Tables~\ref{tab:ablation-lambda} and~\ref{tab:ablation-rho}. 

\begin{table}[!ht]
\centering
\caption{$\lambda$ ablation ($\rho=0.2$ fixed, ACS sex, 10 seeds). $\star$~= default used in all main experiments.}
\label{tab:ablation-lambda}
\resizebox{\columnwidth}{!}{%
\begin{tabular}{lcccc}
\toprule
$\lambda$ & Adv gap & $\Delta_{\text{EO}}$ & ${\le}\,\tau$? & Bound? \\
\midrule
$0.0$       & $4.2\%$  & $20.6\%$ & 10/10 & \cmark \\
$0.1$       & $3.2\%$  & $12.7\%$ & 10/10 & \cmark \\
$0.5$       & $2.8\%$  & $4.1\%$  & 10/10 & \cmark \\
$1.0^\star$ & $2.5\%$  & $4.3\%$  & 10/10 & \cmark \\
$2.0$       & $4.0\%$  & $4.1\%$  & 10/10 & \cmark \\
$5.0$       & $10.8\%$ & $1.3\%$  &  2/10 & \xmark \\
\bottomrule
\end{tabular}}
\end{table}

\begin{table}[!ht]
\centering
\caption{$\rho$ ablation ($\lambda=1.0$ fixed, ACS sex, 10 seeds). $\star$~= default used in all main experiments.}
\label{tab:ablation-rho}
\resizebox{\columnwidth}{!}{%
\begin{tabular}{lcccc}
\toprule
$\rho$ & Adv gap & $\Delta_{\text{EO}}$ & ${\le}\,\tau$? & Bound? \\
\midrule
$0.00$       & $23.0\%$ & $64.5\%$ &  0/10 & \cmark \\
$0.05$       & $3.4\%$  & $6.5\%$  & 10/10 & \cmark \\
$0.10$       & $3.1\%$  & $4.8\%$  & 10/10 & \cmark \\
$0.20^\star$ & $2.5\%$  & $4.3\%$  & 10/10 & \cmark \\
$0.30$       & $2.4\%$  & $3.6\%$  & 10/10 & \cmark \\
$0.50$       & $1.6\%$  & $3.2\%$  & 10/10 & \cmark \\
\bottomrule
\end{tabular}}
\end{table}

\section{DP-SGD vs.\ FFt: Full Results}
\label{app:dpsgd}

Table~\ref{tab:dpsgd} reports the full numerical comparison between DP-SGD and FFt on ACS Income 2018 (California), \texttt{sex} split ($\mathcal{G}_0$=male, $\mathcal{G}_1$=female), averaged over 10 independent seeds.  DP-SGD is trained from scratch on $\mathcal{G}_0$ using Opacus~\citep{opacus21} with $\delta{=}10^{-5}$, $C{=}1.0$, and batch size 256.  FFt warm-starts from the Baseline and fine-tunes on $\mathcal{S}_1$ with EO penalty $\lambda{=}1$ and rehearsal fraction $\rho{=}0.2$, identical to the main experiments.

\begin{table}[h]
\centering
\caption{DP-SGD vs.\ FFt on ACS sex (mean $\pm$ 10 seeds). $\tau=0.1$.}
\label{tab:dpsgd}
\small
\begin{tabular}{lrrrc}
\toprule
Method & Adv gap $\downarrow$ & $\Delta_{\text{EO}}$ $\downarrow$ & $\mathcal{S}_0$ acc & ${\le}\tau$? \\
\midrule
Baseline                       & 14.3\% & 78.0\% & 79.7\% & 0/10 \\
DP-SGD $\varepsilon{=}1$       & 13.2\% & 76.5\% & 78.6\% & 0/10 \\
DP-SGD $\varepsilon{=}2$       & 13.6\% & 77.9\% & 78.9\% & 0/10 \\
DP-SGD $\varepsilon{=}5$       & 13.8\% & 78.4\% & 79.1\% & 0/10 \\
\midrule
FFt (ours)                     & \textbf{2.5\%} & \textbf{4.3\%} & 75.7\% & \textbf{10/10} \\
\bottomrule
\end{tabular}
\end{table}

DP-SGD fails to reduce the adversarial gap at any privacy budget: the gap remains $13$--$14\%$ regardless of $\varepsilon$, and all $0/10$ runs fall below $\tau$.  This is structurally expected: DP-SGD clips and noises per-sample gradients to protect individual records, but does not alter the model's group-level behaviour.  A $\mathcal{G}_0$-trained model will still generalise poorly to $\mathcal{G}_1$, leaving the demographic accuracy gap. Hence the Loss Test adversary's signal here is intact. Tighter $\varepsilon$ (stronger privacy) trades $1.1$ percentage points of $\mathcal{G}_0$ accuracy but yields no DIA protection.

FFt closes the gap to $2.5\%$ (10/10 below $\tau$) because it directly minimises $\Delta_{\text{EO}}$, which by Theorem 1 ($W=1$) bounds the adversary's advantage.  The utility cost is $4.0$ pp on $\mathcal{G}_0$ accuracy.  The bound $\text{Adv} \le \Delta_{\text{EO}}$ holds ($2.5\% \le 4.3\%$), confirming the theoretical guarantee.

Since DP and FFt suppress different aspects of leakage (individual-level vs.\ group-level), they are orthogonal and can in principle be composed.  We leave a systematic composition study to future work.

\section{Per-Run Adversarial Gap Plots}
\label{app:barplots}

Figures~\ref{fig:utk}--\ref{fig:bios} show the per-seed adversarial accuracy gap for UTKFaces and Bias in Bios.  Both follow the same pattern: rehearsal-based FFt consistently reduces the gap, and the bound $\text{Adv} \le \Delta_{\text{EO}}$ holds in every run.

\begin{figure*}[!th]
    \centering
    \includegraphics[width=0.6\linewidth]{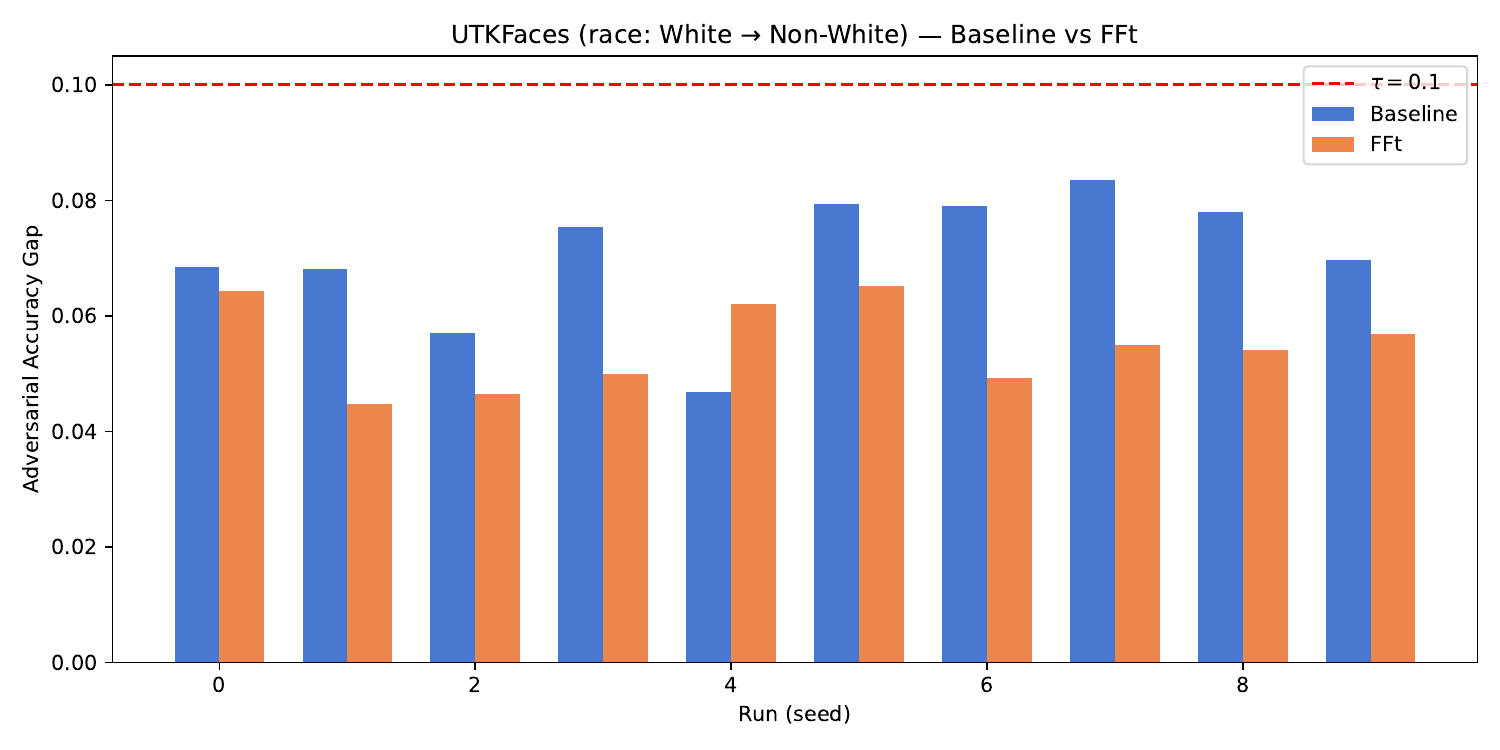}
    \caption{Adversarial accuracy gap for UTKFaces (\texttt{race}: White $\to$ Non-White, mean over 10 runs).}
    \label{fig:utk}
\end{figure*}

\begin{figure*}[!th]
    \centering
    \includegraphics[width=0.6\linewidth]{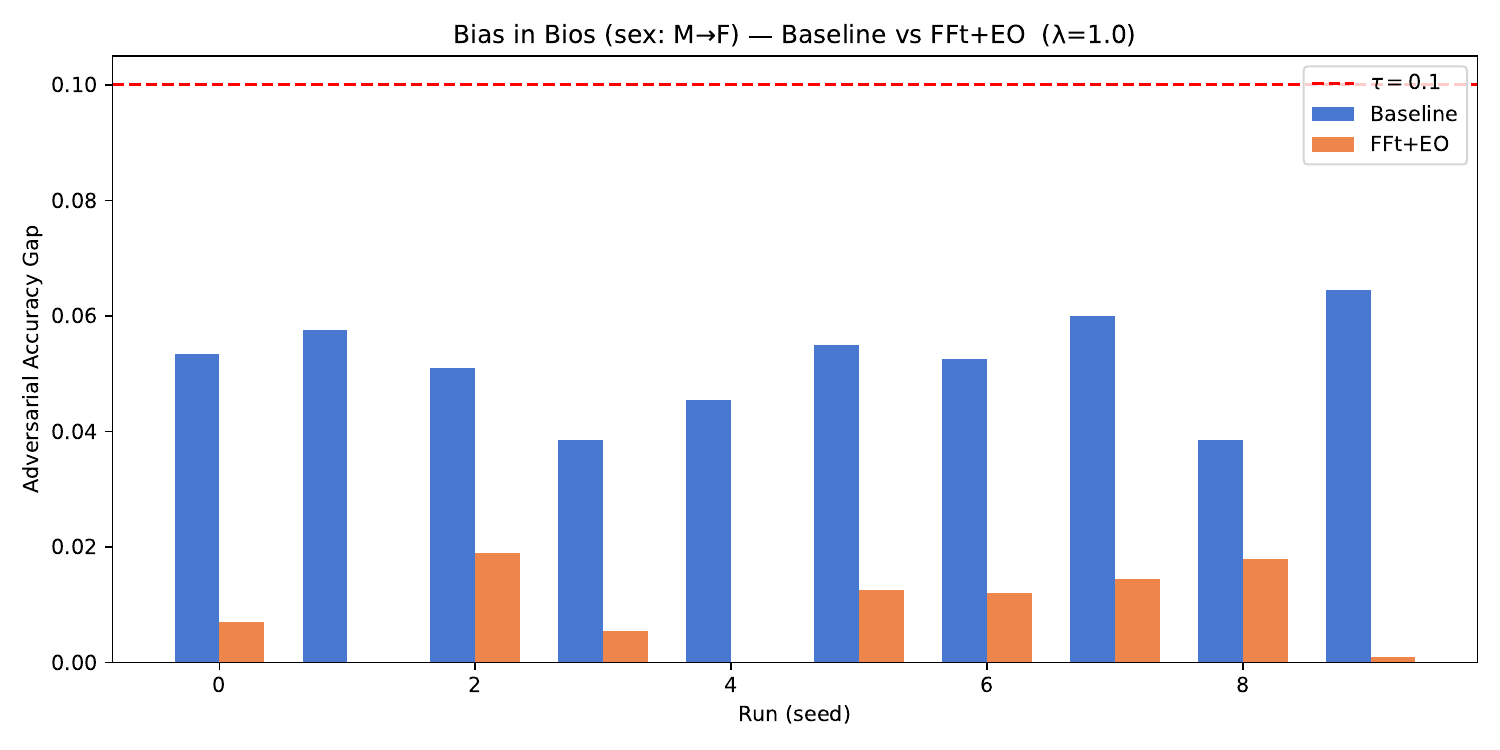}
    \caption{Adversarial accuracy gap for Bias in Bios (\texttt{sex}: Male $\to$ Female, mean over 10 runs).}
    \label{fig:bios}
\end{figure*}

\end{document}